\documentclass[sigconf]{acmart}
\usepackage[ruled,linesnumbered]{algorithm2e}
\usepackage{amsmath}
\usepackage{amsfonts}
\usepackage{graphicx}
\usepackage{subcaption}
\usepackage{tabularx}
\usepackage{multirow}
\AtBeginDocument{%
  }
\renewcommand\footnotetextcopyrightpermission[1]{}
\author{Sihan Bai}
\begin{document}

\title{Pairwise Similarity Distribution Clustering for Noisy Label Learning}

\begin{abstract}
Noisy label learning aims to train deep neural networks using a large amount of samples with noisy labels, whose main challenge comes from how to deal with the inaccurate supervision caused by wrong labels. Existing works either take the label correction or sample selection paradigm to involve more samples with accurate labels into the training process. In this paper, we propose a simple yet effective sample selection algorithm, termed as Pairwise Similarity Distribution Clustering~(PSDC), to divide the training samples into one clean set and another noisy set, which can power any of the off-the-shelf semi-supervised learning regimes to further train networks for different downstream tasks. Specifically, we take the pairwise similarity between sample pairs to represent the sample
structure, and the Gaussian Mixture Model~(GMM) to model the similarity distribution between sample pairs belonging to the same noisy cluster, therefore each sample can be confidently divided into the clean set or noisy set. Even under severe label noise rate, the resulting data partition mechanism has been proved to be more robust in judging the label confidence in both theory and practice. Experimental results on various benchmark datasets, such as CIFAR-10, CIFAR-100 and Clothing1M, demonstrate significant improvements over state-of-the-art methods.
\end{abstract}


\maketitle

\section{Introduction}

\begin{figure}[t]
	\begin{center}
		\includegraphics[width=0.95\linewidth]{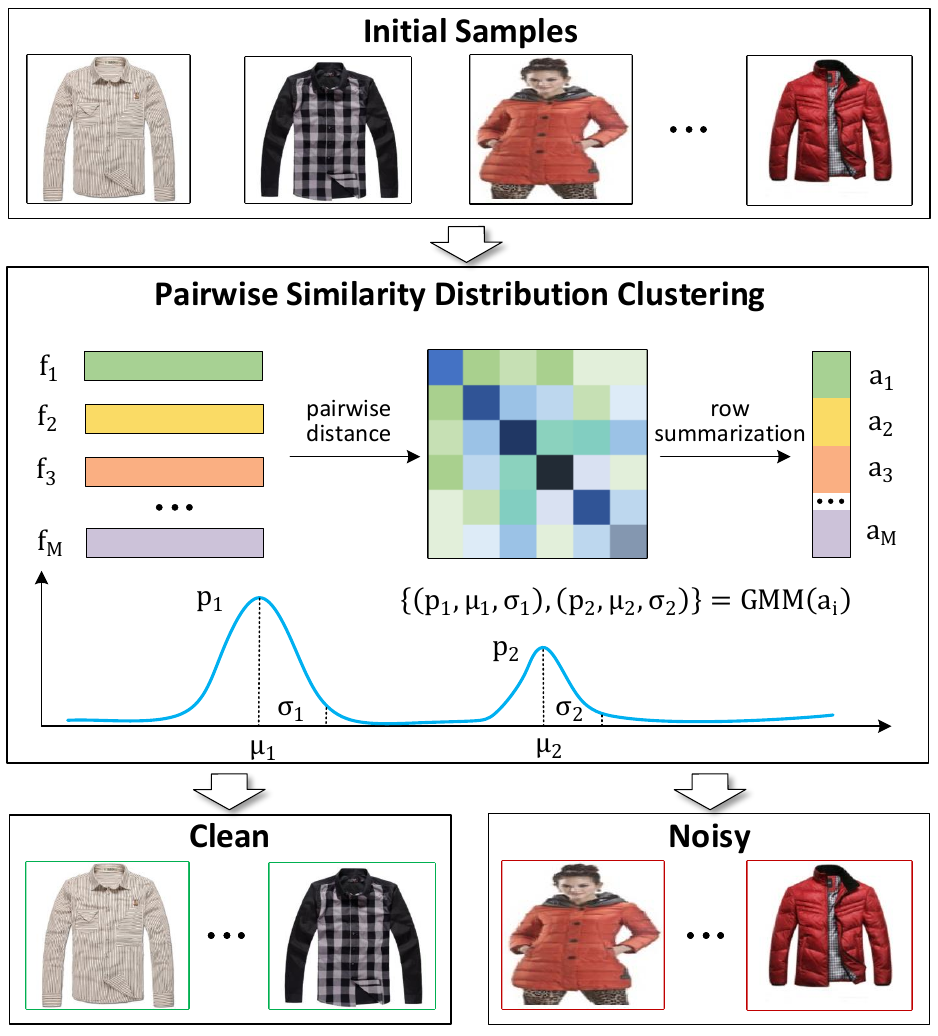}
	\end{center}
	\caption{Illustration of sample selection through pairwise similarity distribution clustering. For each group with the same label, we first calculate the cosine distance between all sample pairs, then summarize the distribution matrix by row and divide all the samples into two groups using gaussian mixture model.}
  \label{fig1}
\end{figure}

The remarkable success of deep learning is largely attributed to the training of Deep Neural Network~(DNN) using a large datasets with human annotated labels~\cite{Zhou_Wang_Wang:2017,He_Zhang_Ren:2016,Zhou_Wang_Wang:2023}. However, it is not only  labor-expensive but also time-consuming to label extensive data with high-quality annotations~\cite{Zhou_Wang_Shu:2021}.To overcome this problem, noisy label learning~\cite{Shu_Xie_Yi:2019} has been widely studied by worldwide researchers, which aims to train DNNs only with a large amount of samples with noisy labels. Because of the significant number of parameters, the DNNs are very easy to overfit the noisy labels by learning complex mapping functions~\cite{Song_Kim_Park:2022}. Therefore, how to involve more samples with accurate
labels into the training process has become a critical issue in noisy label learning.

In the past few years, extensive works have been done in the field of noisy label learning, which could be simply divided into two categories, \emph{i.e.}, label correction~\cite{Shu_Xie_Yi:2019,Zheng_Awadallah_Dumais:2021,Wang_Hua_Kodirov:2021} and sample selection~\cite{Wu_Shu_Xie:2021,Ren_Zeng_Yang:2018,Zhang_Wu_Chen:2020}. In particular, the label correction methods try to learn a transition matrix between network's predictions and noisy labels, in which the transition matrix will revise the wrong gradients indicated by noisy labels. For example, the Meta Soft Label Corrector~(MSLC)~\cite{Shu_Xie_Yi:2019} takes extra clean annotations to learn a transition matrix by utilizing the original targets and dynamic predictions. What's different, the sample selection methods aim to choose more samples with clean labels into the training process, in which the sample reweighting strategy is often used to drop out samples based on different priori knowledge. For example, the Self-Paced Robust Learning~(SPRL) algorithm~\cite{Zhang_Wu_Chen:2020}  trains the DNN in a process from more reliable samples to less reliable ones under the supervision of well labeled data. Even though great progress has been achieved in noisy label learning, all the existing methods still suffer from improving the quality of supervision under severe label noise rate.

Recently, the mainstream methods usually formulate the noisy label learning as a semi-supervised learning problem~\cite{Karim_Rizve_Rahnavard:2022,Li_Richard_Steven:2020}, in which the training samples are first divided into one clean set and another noisy set, and then the ones in clean set are taken for supervised learning while the ones in noisy set are taken for unsupervised learning. From this point of view, this line of methods actually belongs to the sample selection category. These methods~\cite{Karim_Rizve_Rahnavard:2022,Berthelot_Carlini_Goodfellow:2019,Wang_Li_Guo:2021} have achieved the state-of-the-art performance in noisy label learning, whose success is mainly because of it's a relatively easy thing to divide the training samples into one clean set and another noisy set, therefore the potential of samples with clean labels can be fully explored in the supervised learning step. For example, the recent U{\footnotesize{NI}}form selection and C{\footnotesize{ON}}trastive learning~(U{\footnotesize{NI}}C{\footnotesize{ON}}) method~\cite{Karim_Rizve_Rahnavard:2022}  takes the Jensen-Shannon Divergence~(JSD) as a metric to conduct data partition, in which the pureness of clean set can reach about 90\% in the training process. To further improve the performance, the key of these methods lies on how to accurately divide the training samples into clean set and noisy set.

To achieve this goal, different divergence metrics are formulated to construct the clean set and noisy set. For example, the Jensen-Shannon Divergence~(JSD)~\cite{Karim_Rizve_Rahnavard:2022}, Cross-Entropy~\cite{gui2021understanding,Li_Richard_Steven:2020} are widely used for data partition, where the samples with lower losses are divided into the clean set, while the samples with higher losses are divided into the noisy set. Some methods manually set a cut-off threshold between small and large losses~\cite{gui2021understanding}, others use clustering methods like Beta Mixture Model~(BMM)~\cite{Arazo_Ortego_Albert:2019} and Gaussian Mixture Model~(GMM)~\cite{Li_Richard_Steven:2020} or an automatically calculated threshold~\cite{Karim_Rizve_Rahnavard:2022} to automatically separate losses. 
Because these methods take the small loss criterion for data partition, they are very easy to overfitting noisy labels when the noise rate is high in the original training data. To alleviate this issue, two critical points need to be addressed in the resulting work:~(1) The adopted loss should be robust to different noise rates; and~(2) The divergence metric should have a wide adaptability to different rates of noisy labels.

In this paper, we propose a novel Pairwise Similarity Distribution Clustering~(PSDC) method for noisy label learning, which is effective to divide the training samples into one clean set and another noisy set. Therefore,  the resulting samples can be further taken to learn discriminative features representation in a supervised and unsupervised manner, respectively. In particular, we compute the pairwise similarity between sample pairs to represent the sample structure in each noisy cluster, which can act as important prior information for candidate sample separation. Because the pairwise sample structure has no direct relationship with noisy labels, it can overcome the drawback of small loss criterion in data partition. Besides, we take the Gaussian Mixture Model~(GMM) to model the similarity distribution between sample pairs belonging to the same noisy cluster, therefore each sample can be confidently divided into the clean set or noisy set. As shown in Figure~\ref{fig1}, the samples of ``shirt'' can be picked out to learn a discriminative feature representation in supervised manner, while the samples not belonging to ``shirt' can be taken to learn a robust feature representation in an unsupervised manner. Even under severe specific types of label noise rate, we have proved that the resulting data partition mechanism is robust in judging the label confidence in both theory and practice.

The main contributions of this work is highlighted as follows:
\begin{itemize}
    \item We propose a novel Pairwise Similarity Distribution Clustering method for noisy label learning, which takes the pairwise sample structure and Gaussian Mixture Model to improve the accuracy of data partition.
    \item We present a clear theoretical analysis to Jensen-Shannon Divergence, Cross-Entropy Criterion and Gaussian Mixture Model, which indicates that our method has a wide tolerance range to noise rate. 
\item We conduct extensive experiments on CIFAR-10, CIFAR-100 and Clothing1M datasets with different types and noise rates, and achieve the state-of-the-art results.
\end{itemize}

\section{related work}
In this section, we briefly review the related works from two aspects, \emph{i.e.}, label correction and sample selection, which are introduced in the following paragraphs.

\textbf{Label Correction}. 
For label correction methods, a noisy label is usually refurbished to prevent network overfitting to false labels. For example, Bootstrapping~\cite{reed2015training} first proposes the concept of label correction to update the target labels of samples in the training process. Recently, the iterative methods, such as Joint Optimization Framework~(JOF)~\cite{tanaka2018joint} and Online Label Smoothing~(OLS)~\cite{Zhang_Jiang_Hou:2021}, relabel samples based on the predictions on networks. What's different, the loss correction~\cite{Patrini2016MakingDN} methods often estimate the noise transition matrix, which represents the probabilities that clean labels flip into noisy labels. For example, the Gold Loss Correction~(GLC)~\cite{hendrycks2018glc} estimates the noise transition matrix using a set of training samples with clean labels, whose performance is closely related with the initial representation of network. To alleviate this issue, the Meta Label Correction~(MLC)~\cite{zheng2021meta} introduces a meta learning framework to learn the noise transition matrix, in which an adaptive meta corrector is learned to cope with the enhancement of representation ability in the whole training process. What's more, some other works focus on how to involve more samples with noisy labels into the training process. For example, work in~\cite{pmlr-v97-song19b} selectively refurbish and exploit those samples whose labels can be corrected with high precision, which can prevent the risk of noise accumulation by gradually increasing the number of training samples. Besides, Self-Ensemble Label Correction~(SELC)~\cite{lu2022selc} uses ensemble predictions formed by an exponential moving average of network's outputs to update the original noisy labels. Self-Evolution Average Label~(SEAL)~\cite{chen2020classconditional} presents a simple yet effective algorithm on instance dependent noise, in which both theoretical and empirical evidences are given to show its robustness to instance dependent noise.

\textbf{Sample Selection}. For sample selection methods, different prior knowledge is usually modeled to choose more reliable samples to train networks. For example, most sample selection based methods take the phenomenon that DNN learns simple patterns before fitting label noise~\cite{article} as a prior knowledge, thererfore they first choose samples with samll losses to train network and then add more samples with large losses into the training process. What's different, Meta Weight Network~(MWN)~\cite{Shu_Xie_Yi:2019} tries to learn a sample selection criterion, in which a set of samples with clean labels are taken as meta data to learn how to choose samples with clean labels. Recently, the mainstream methods take semi-supervised learning~\cite{Berthelot_Carlini_Goodfellow:2019} and co-training technologies~\cite{coteaching} to train networks, in which the selected clean samples are treated as labeled data and noisy samples are treated as unlabeled data~\cite{ijcai2021p340}. For example, the well known DivideMix~\cite{Li_Richard_Steven:2020} models the losses of samples with GMM, which can dynamically divide the training data into a labeled set with clean samples and an unlabeled set with noisy samples. Besides, the recent U{\footnotesize{NI}}C{\footnotesize{ON}}~\cite{Karim_Rizve_Rahnavard:2022} uses Jensen-Shannon divergence to select samples, which has achieved the state-of-the-art results in noisy label learning. Although these methods have achieved the promising performance, they are sensitive to the training data which has a high rate of noisy labels. To the best of our knowledge, the underying reason is that the noisy labels are directly used to measure their own credibility, whose idea is mainly based on the phenomenon that networks first learn samples with clean labels and then learn samples with wrong labels~\cite{article}. Unlike the mainstream sample selection methods, we aim to explore the structure between sample pairs in data partition, and provide the insight why it is more robust than using noisy labes to divide samples into clean and noisy set.

\section{method}
\subsection{Preliminaries}
Let $\mathcal{X}$ denote the instance space, and $\mathcal{Y}$ denote the label space, such that for each instance $x\in\mathcal{X}$, there exists a corresponding label $y\in\mathcal{Y}$. We denote the noise-free dataset as $\mathbb{D}=\{\mathcal{X},\mathcal{Y}\}=\{(x_i,y_i)\}_{i=1}^{N}$, where $x_i$ represents an image, $y_i$ represents its corresponding label, and $N$ denotes the total number of training samples. Let us consider a measurable function $\mathcal{C}: \mathcal{X}\rightarrow \mathcal{Y}$, which maps the data samples to their corresponding real labels. Next, we consider datasets where the given label may be corrupted and the labeling is not entirely accurate. We base our approach on the class-conditional noise assumption $p(\tilde{y}|y,\boldsymbol{x})=p(\tilde{y}|y)$, where $\tilde{y}$ represents the corrupted label~\cite{condition}. In practice, we typically have a training dataset $\mathbb{\tilde{D}}=\{\mathcal{X},\mathcal{\tilde{Y}}\}=\{(x_i,\tilde{y}_i)\}_{i=1}^{N}$ where the label set defaults to the noise label set. \par

For a k-class classification problem, we begin by initializing a DNN model with a feature extractor $f(\bullet;\theta)$. After the feature extractor, there is a classification layer, $h : f(\mathcal{X};\theta) \rightarrow \mathbb{R}^{N \times k}$ and a projection head, $g : f(\mathcal{X};\theta)\rightarrow \mathbb{R}^{N \times m}$, where $m$ is the dimension; $g$ is used for contrastive learning. We minimize a loss function, $l: \mathbb{R}^{N\times k}\times \mathcal{\tilde{Y}}\rightarrow \mathbb{R}^{N}$ to train with the given labels on the training set $\tilde{\mathbb{D}}$. 
We face a sample selection problem where we need to partition a training set $\tilde{\mathbb{D}}$ into a clean subset, $\mathbb{D}_{clean}$, and a noisy subset, $\mathbb{D}_{noisy} = \tilde{\mathbb{D}}\backslash \mathbb{D}_{clean}$. Then, $\mathbb{D}_{clean}$ is used for supervised training, while $\mathbb{D}_{noisy}$ is utilized for unsupervised training without using the corresponding noisy ground-truth labels. This approach is a standard semi-supervised method where pseudo-labels are generated for the examples in $\mathbb{D}_{noisy}$.
This section introduces our proposed method for learning with noisy labels. It includes a sample selection module and a semi-supervised learning module, which are introduced in the following paragraphs.
\subsection{Pairwise Similarity Distribution Clustering}
Consider the features extracted by the backbone with a projection head $\mathbb{G} = g(f(\mathcal{X},\theta),\phi)$. We then partition $\mathbb{G}$ based on the given labels in $\mathcal{\tilde{Y}}$ as follows:
\begin{equation}
    \mathbb{G} = \{\mathbb{G}_i\}_{i=1}^{k},
\end{equation}
where $k$ denotes the number of classes, and the samples in each $\mathbb{G}_i$ share the same given label.
We approach the problem as a binary classification task for each $\mathbb{G}_i$, given that the label for each sample in the set is the same. Thus, we only need to consider the feature differences of each sample. Clean samples that belong to the label set depict the same thing, so they have similar features. However, noise samples that should not be in the set depict different things, thus they do not share similar features with clean samples. To gauge pairwise similarity of samples in the set $\mathbb{G}_i$, we calculate the cosine similarity and generate an affinity matrix defined as:
\begin{definition}
    \textbf{Affinity Matrix}: Let $A^i \in \mathbb{R}^{n\times n}, i=\{1,2,...,k\}$ denote the affinity matrix of $\mathbb{G}_i$, $i=1,2,...,k$ , where k is the class number of training dataset, n is the sample number of $\mathbb{G}_i$. Row $p$ of $A^i$ denotes the similarity measure between sample $x_p$ and other samples. 
\end{definition}
To facilitate the representation of theory, some concept is proposed here for noisy data:
\begin{definition}
    \textbf{Submerged}: Assume that the average affinity of  clean samples $a_p = \{a_{p_1},...,a_{p_n}\}$ are a sequence of i.i.d. random variables with expectation $\mu_p$ and variance $\sigma_{p}^{2}$, the average affinity of noisy samples $a_q = \{a_{q_1},...,a_{q_{n}}\}$are a sequence of independent random variables satisfying Liapunov Condition~\cite{adams2009life},
    expectation $\mu_q$ and variance $\sigma_{q}^{2}$. 
    if
    \begin{equation}
        \sum_{i=p_1}^{p_n} a_i < \sum_{i=q_1}^{q_n} a_i,
    \end{equation}
    then the clean samples are submerged by noisy samples. 
\end{definition}

Our research introduces the first theory, as described below:
\begin{theorem}\label{1}
Consider two pairs of samples, ${(x_p,\tilde{y}),(x_q,\tilde{y})}$, randomly selected from $\mathbb{G}_i$
, with their respective indices in the affinity matrix $A^i$ being $p$ and $q$. Given the following conditions:
\begin{enumerate}
\item $\mathcal{C}(x_p) \neq \tilde{y}$ and $\mathcal{C}(x_q) = \tilde{y}$;
\item clean samples are not submerged by noise samples;
\item clean samples and noise samples obey different distributions.
\end{enumerate}
Then the mean value of row $p$ on the affinity matrix $A^i$ follows a Gaussian distribution with mean $\mu_p$ and the mean value of row $q$ in the affinity matrix $A^i$ follows a Gaussian distribution with mean $\mu_q$, where $\mu_q < \mu_p$.  
\end{theorem}
Theorem \ref{1} states that a sample is classified as a noisy sample when its affinities with all other samples are small. Building on this theory, we sum the matrices $A^i$ by rows to obtain a series of column vectors $a^i \in \mathbb{R}$, where $i = \{1,2,...,k\}$. The values of $a^i$ represent the sum of affinities between clean and noise samples. According to the central limit theorem, the mean value of $a^i$ for both clean and noise samples follows a normal distribution. We omit the step of dividing by $n$ because a series of samples that corresponds to a normal distribution will also correspond to a normal distribution when multiplied by the same value simultaneously. Therefore, we believe that the samples belonging to a Gaussian distribution with a higher mean value are clean samples, while the others are noise samples. Based on this, we adopted the following steps to conduct sample screening: We represent the training sample feature set by $\mathbb{G}$, where for a feature set $\mathbb{G}_k=\{g_i\}_{i=1}^l$ with $l$ features and a label category of $k$, we calculate the cosine distribution:
\begin{equation}
    A^i_{j,z} = \mathrm{Cosine}\left( g_z, g_j\right)=\dfrac{ g_z\cdot g_j}{|| g_z||\times|| g_j||}, 
\end{equation}
where $g_z,g_j$ are two fearures in $\mathbb{G}_k$. We then represent $A^i$ as column vectors $A^i=[a_1\ a_2\ ...\ a_l]$. By summing these column vectors
\begin{equation}
    a^i = \sum_{index=1}^l a_{index},
\end{equation}
we obtain a value of $a^i$ that follows two different normal distributions. To differentiate between samples, we fit a two-component GMM using the Expectation-Maximization algorithm. For each sample, GMM can give a posterior probability that it belongs to two Gaussian distributions and the mean value of each normal distribution.
We can then select clean samples based on their posterior probability values being higher than a cutoff value of $d_{cutoff}$ in order to improve accuracy. The detailed the sample selection algorithm is presented in Algorithm~\ref{al1}.

\begin{algorithm}
\caption{Pairwise Similarity Distribution Clustering}
\label{al1}
  \SetAlgoLined
  \SetKwInOut{Input}{Input}\SetKwInOut{Output}{Output}
  \Input{training set $\mathbb{\tilde{D}}$, 
  features $\mathbb{G}=\{g_i\}_{i=1}^N$ of training set $\mathbb{\tilde{D}}=\{\mathcal{X},\mathcal{\tilde{Y}}\}=\{(x_i,\tilde{y}_i)\}_{i=1}^{N}$, , labels $\tilde{\mathcal{Y}}$ of training set $\mathbb{\tilde{D}}$, number of class k
  }
  \For{i=1 to N}
{put $(g_i,\tilde{y}_i)$ into $\mathbb{G}_{\tilde{y}_i}$ 
}
\For{i = 1 to k}{
    
    l = length($\mathbb{G}_k$)\\
    \For{j = 1 to l}{
        \For {z = 1 to l}{
        $A^i_{j,z} = cosine(g_z,g_j)$
        }
    }
    $a^i =$ sum $A^i$ by row\\
    $\{(p_1,\mu_1,\sigma_1),(p_2,\mu_2,\sigma_2)\} = GMM(a^i)$\\
    \For{j = 1 to l}{
    \If{$((p_{1,j}>p_{2,j})\wedge (\mu_{1,j}>\mu_{2,j}) \wedge (p_{1,j}>d_{cutoff}))\vee((p_{1,j}<p_{2,j})\wedge (\mu_{1,j}<\mu_{2,j})\wedge (p_{2,j}>d_{cutoff}))$}{put $(x_j,\tilde{y}_j)$ into $\mathbb{D}_{clean}^i$
    }
    \Else{
    put $(x_j,\tilde{y}_j)$ into $\mathbb{D}_{noisy}^i$
    }
    }
}
$\mathbb{D}_{clean}=\bigcup_{i=1}^{k} \mathbb{D}_{clean}^i$\\
$\mathbb{D}_{noisy}=\bigcup_{i=1}^{k} \mathbb{D}_{noisy}^i$\\
\Output{
$\mathbb{D}_{clean},\mathbb{D}_{noisy}$
}
\end{algorithm}

In practice, the features extracted by network usually tend to be inaccurate at the beginning, therefore we jointly take the uniform selection criterion~\cite{Karim_Rizve_Rahnavard:2022} to help data partition at early stages. In particular, the selecting results given by JSD are utilized when the number of joint samples chosen by PSDC and JSD is lower than the 0.8 times of that chosen by JSD. With the iteration going on, the representation capability of network will become stronger, therefore we simply use our PSDC to conduct sample selection.  As shown in Figure \ref{fig3}, we illustrate some sample selection examples on the Clothing1M, CIFAR-100, and CIFAR-10 datasets, in which the noise rates are set to be $3/7$, $2/7$, and $2/7$, respectively. Due to the powerful capability of our PSDC, the average purity of noisy set and clean set can reach $6/7$, which is a strong guarantee to any semi-supervised learning algorithm.

\begin{figure}[t]
	\begin{center}
		\includegraphics[width=1\linewidth]{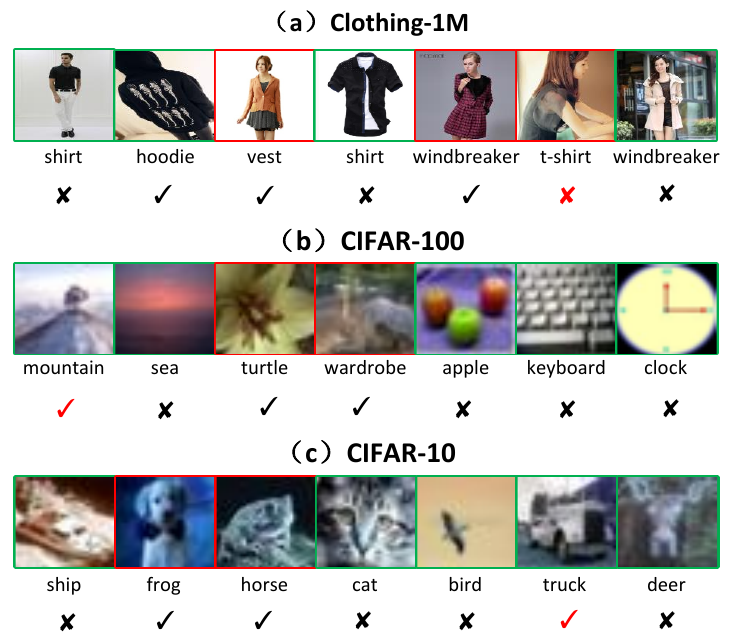}
	\end{center}
	\caption{Illustration of sample selection examples by our PSDC on the Clothing1M, CIFAR-100, and CIFAR-10 datasets. In particular, each image includes an assigned label at the bottom indicating whether it is clean or noisy. The clean labels are surrounded by green borders, while the noisy labels are bordered in red. The selection results are indicated by checkmark and cross, in which the noisy label is marked by cross and the clean label is marked by checkmark. Besides, the red checkmark or cross means that our PSDC makes wrong data partition to this sample. }
  \label{fig3}
\end{figure}

\begin{figure*}[t]
	\begin{center}
		\includegraphics[width=0.95\linewidth]{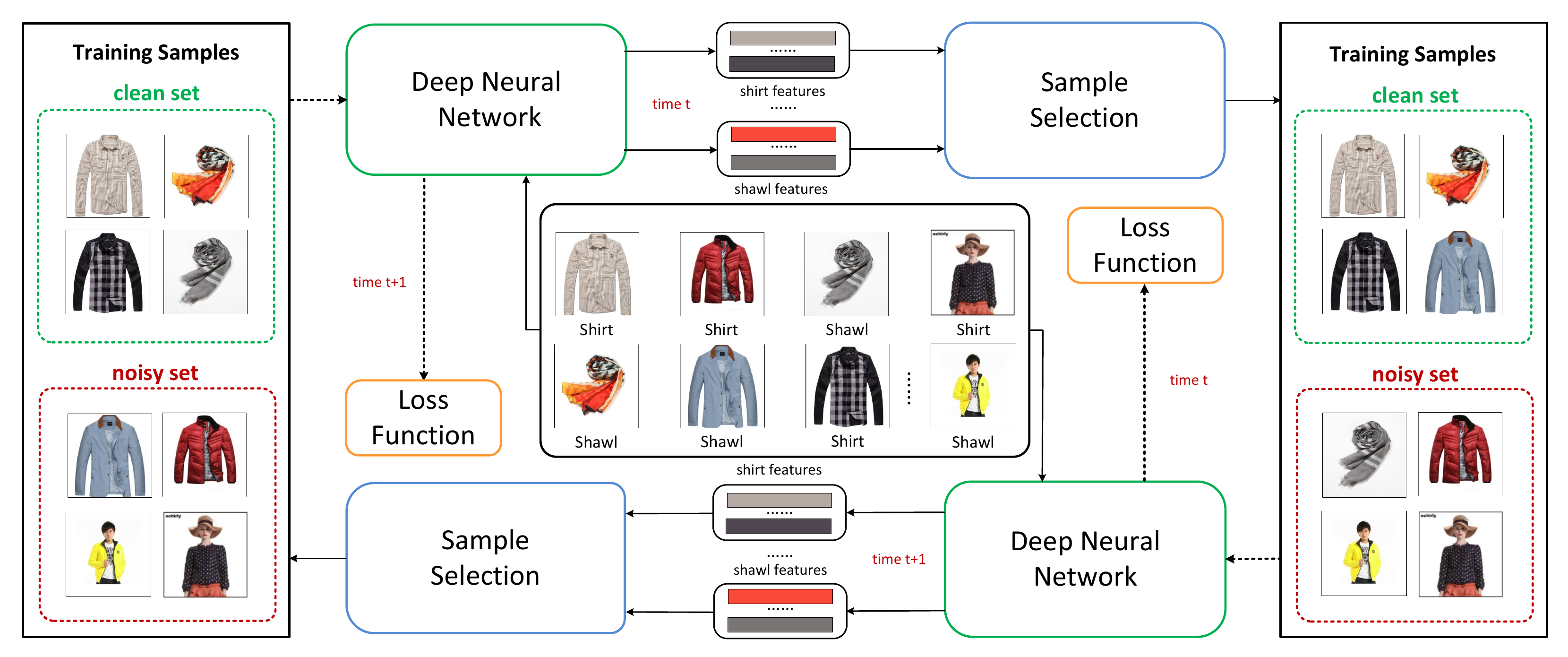}
	\end{center}
        \vspace{-0.2cm}
	\caption{Illustration of the semi-supervised training framewor. At each time $t$, the current network is first used to extract features for all training samples. Then, these features are taken to divide the training sample into clean set and noisy set using our PSDC algorithm. Finally, the sample selection results are further taken to power the semi-supervised learning regime. Once the current network is updated at time $t+1$, it is used to conduct sample selection in a new round. With more and more samples are correctly divided into clean set and noisy set, the network will also become powerful enough in the semi-supervised training manner. }
  \label{fig2}
\end{figure*}

\subsection{Theoretical analysis}
Various techniques have been employed in prior works to create clean and noisy subsets. Amongst them, the recent state-of-the-art method U{\footnotesize{NI}}C{\footnotesize{ON}}~\cite{Karim_Rizve_Rahnavard:2022} uses JSD to select samples. We conducted an analysis to compare the theoretical scope and applicability of this method with our own. 
Assuming class-conditional noise, the label corruption can be expressed by a noise transition matrix $T\in\mathbb{R}^{k\times k}$. Here, $T_{ij}=p(\tilde{y}=j|y=i)$ indicates the probability of flipping a class-i example into a class-j example. The noisy data distribution satisfies $p(\boldsymbol{x},\tilde{y}) = \sum_{i=1}^kp(\tilde{y}|y=i)p(\boldsymbol{x},y=i)$.\par

We assume the presence of a backbone that includes a classification layer, $h(f(\boldsymbol{x};\Theta)):\mathcal{X}\rightarrow\mathbb{R}^{N\times k}$ with output $g(\boldsymbol{x};\Theta)=[\hat{p}_1(\boldsymbol{x}),...,\hat{p}_k(\boldsymbol{x})]\in\mathbb{R}^k$, where $\boldsymbol{x}\in \mathcal{X}$ is a random sample. For simplicity, we denote $h(f(\boldsymbol{x};\Theta))$ as $h(\boldsymbol{x})$ and omit the $\Theta$ parameter. The softmax output of $h$ for $(\boldsymbol{x},y)$ is $\left[\hat{p}_{1}\left(\boldsymbol{x}\right), \ldots, \hat{p}_{k}\left(\boldsymbol{x}\right)\right]=\left[T_{h\left(\boldsymbol{x}\right) 1}, \ldots, T_{h\left(\boldsymbol{x}\right) k}\right]$. 
\begin{theorem}\label{2}
Consider two pairs of randomly selected samples, $(\boldsymbol{x}_1,\tilde{y})$ and $(\boldsymbol{x}_2,\tilde{y})$, with the same observed label from the set $\{\mathcal{X},\mathcal{\tilde{Y}}\}$. If the following conditions hold:
\begin{enumerate}
    \item $\mathcal{C}(\boldsymbol{x}_1)=\Tilde{y}$ and $\mathcal{C}(\boldsymbol{x}_2)\neq\Tilde{y}$;
    \item the noise transition matrix T of $\{\mathcal{X},\mathcal{\tilde{Y}}\}$ satisfies the diagonally-dominant condition $T_{ii}>max\{max_{j\neq i}T_{ij},max_{j\neq i}T_{ji}\}$,$\forall i$;
    \item noise type is uniform and noise rate is under 1, or noise type is pairwise and noise rate is under 0.5, or noise type is structured and noise rate is under 0.5. 
\end{enumerate}
are satisfied, then $JSD(h(\boldsymbol{x}_1),\Tilde{y})<JSD(h(\boldsymbol{x}_2),\Tilde{y})$, where JSD denotes Jensen-Shannon divergence. 
\end{theorem}
Theorem \ref{2} provides a definite order relationship for specific types of noise. However, for other types of noise, it is difficult to analyze, which raises concerns about the reliability of the method in realistic noisy environments. Additionally, because this method relies on labels, overfitting to noisy labels can cause the JSD for clean and noisy samples to become increasingly difficult to distinguish.

Our method theorem \ref{1} provides a submerged condition for effectiveness. This means that as long as the noise is not so severe that multiple samples differing from the clean category and resembling the clean category appear simultaneously in the same category of the dataset, or the number of a class of noisy samples exceeds that of the clean category, our method can effectively detect the noise. Additionally, Lyapunov condition ensures that the random variables in the random variable column are ``uniform'', ``equal'', with each random variable being ``insignificant''. This condition is an inference of the Lindeberg condition:
\begin{equation}
    \lim_{n\to\infty}\frac{1}{B_n^2}\sum_{i=1}^n E((X_i-\mu_i)^2I[|X_i-\mu_i|\ge\epsilon B_n])=0.
\end{equation}
In simple terms, the Lyapunov condition requires that the sum of variances for all random variables $B_n$, as shown in Eq.~\eqref{eq1},
\begin{equation}
\label{eq1}
     B_n = Var(a_{q_1}) + Var(a_{q_2}) + ... + Var(a_{q_{n}}),
\end{equation} 
be sufficiently large while the effect of any individual random variable, $a_i$, on the total variance is small. This ensures that no single random variable dominates the change in total variance. With the DNN training convergence, affinity matrix values, $A^i$, become stable, and as a result, most datasets satisfy the Lyapunov condition. Moreover, pairwise similarity distribution clustering is not a label-dependent method, its effectiveness solely depends on the accuracy of feature extraction. This reduces the impact of network overfitting on noise in sample selection, giving our method an advantage over loss-based methods. Therefore, our method has a wider range of theoretical applications and greater clarity.

In summary, grouping samples is a better strategy than selecting all samples directly because it fits the theoretical framework presented in Theorems \ref{1} and \ref{2}, in which a partial order relation is established for samples with identical ground truth labels but different clean labels. In our work, if the group mechanism is missing and similarity distribution is calculated directly among all samples, then similar samples will not dominate among all samples, particularly when the dataset is large and there are several non-similar sample pairs. Under these conditions, the similarity of all clean samples will be overshadowed by the similarity of noisy samples, and PSDC will be invalid. \par 
Proofs of theorem \ref{1} and theorem \ref{2} can be seen in the supplementary materials.

\subsection{Semi-supervised Training}
Once the training samples are divided into clean set and noisy set, any of the off-the-shelf semi-supervised learning methods can be further applied to train DNN. In particular, Figure~\ref{fig2} illustrates our training and sample updating process. At time $t$, the previously trained DNN is fed with the training set, and their features are extracted. Afterward, sample features are grouped by their ground truth labels, and PSDC selects each group of samples. The resulting clean set $\mathbb{D}_{clean}$, and noisy set $\mathbb{D}_{noisy}$, are subsequently used for semi-supervised training of the DNN at time $t+1$.

Inspired by DivideMix~\cite{Li_Richard_Steven:2020}, we train two networks simultaneously. Algorithm~\ref{al2} depicts the algorithm for Semi-Supervised Training. At each epoch, a network accepts the $\mathbb{D}_{clean}$ and $\mathbb{D}_{noisy}$ as labeled dataset and unlabeled dataset. For each mini-batch, MixMatch~\cite{berthelot2019mixmatch} with contrastive learning SimCLR~\cite{chen2020simple} is used for semi-supervised training. By training a model using self-extracted features and dividing the data, it might lead to confirmation bias~\cite{tarvainen2018mean}. Therefore, co-teaching~\cite{han2018coteaching} is implemented to prevent error accumulation. The features extracted by one network are used for sample selection in the other network. The two networks are kept distinct due to varying random parameter initialization, training data selection, and training shuffle. \par
During the training process, we create two types of augmented datasets - strongly augmented and weakly augmented - from both labeled and unlabeled datasets. The weakly augmented unlabeled dataset is utilized for guessing pseudo-labels, while the weakly augmented labeled dataset is used for label co-refinement. The strongly augmented unlabeled dataset is employed to compute the contrastive loss. Additionally, MixUp~\cite{zhang2018mixup} is performed on the ground-truth labeled samples and pseudo-labeled samples from the labeled and unlabeled datasets, respectively, to produce two augmented datasets, namely $\hat{\mathcal{X}}$ and $\hat{\mathcal{U}}$. \par
The semi-supervised losses are generated after the MixUp operation as follows: 
\begin{equation}
    \mathcal{L}_{\mathcal{X}}=\frac{1}{|\hat{\mathcal{X}}|}\sum_{\mathbf{x},\mathbf{p}\in\hat{\mathcal{X}}}\mathrm{H}(\mathbf{p},\mathbf{h}(\mathbf{f}(\mathbf{y}\mid\mathbf{x};\theta);\phi)),
\end{equation}
\begin{equation}
    \mathcal{L}_\mathcal{U}=\dfrac{1}{|\hat{\mathcal{U}}|}\sum_{\mathbf{u},\mathbf{q}\in\hat{\mathcal{U}}}\|\mathbf{q}-\mathbf{h}(\mathbf{f}(\mathbf{y}\mid\mathbf{u};\theta);\phi)\|_2^2,
\end{equation}
where $p$ refers to the co-refinement labels and $q$ are the pseudo labels, $H(\bullet,\bullet)$ is the cross-entropy, y are ground-truth labels. Moreover, we applied regularization to prevent the single-class assignment of all examples~\cite{tanaka2018joint}:
\begin{equation}
    \mathcal{L}_{R}=\sum_c\pi_c log\big(\frac{1}{\frac{1}{|\mathcal{X}+\mathcal{U}|}\sum_{\mathbf{x}\in|\mathcal{X}+\mathcal{U}|}\mathbf{h}(\mathbf{f}(\mathbf{x};\theta);\phi)}\big).
\end{equation}
We introduce a contrastive loss for the unlabeled dataset, using projected features $z_i$ and $z_j$ of the augmented samples from the unlabeled dataset $x_i$ and $x_j$. The contrastive loss function is expressed as: 
\begin{equation}
    \ell_{i,j}=-\log\frac{\exp(sim(\mathbf{z}_i,\mathbf{z}_j)/\kappa)}{\sum_{b=1}^{2B}\mathbf{1}_{b\neq i}\exp(sim(\mathbf{z}_i,\mathbf{z}_b)/\kappa)},
\end{equation}
\begin{equation}\label{8}
    {\mathcal L}_{\mathcal C}=\frac{1}{2B}\sum_{b=1}^{2B}[\ell_{2b-1,2b}+\ell_{2b,2b-1}],
\end{equation}
where $\mathbf{1}_{b\neq i}$ is an indicator function, $\kappa$ is a temperature constant and we set $\kappa=1$ in our work, $B$ is the number of samples in a mini-batch, $sim(z_i,z_j)$ is the cosine similarity between $z_i,z_j$.
Finally, we calculate $\mathcal{L}_{T}$, which is the total loss function to minimize. The total loss function we minimize is 
\begin{equation}
    \mathcal{L}_{T}=\mathcal{L}_{\mathcal{X}}+\lambda_{\mathcal{U}}\mathcal{L}_{\mathcal{U}}+\lambda_{R}\mathcal{L}_{R}+\lambda_\mathcal{C}\mathcal{L}_\mathcal{C},
\end{equation}
where $\lambda_{\mathcal{U}}, \lambda_R,\lambda_\mathcal{C}$ are loss coefficients. \par

\begin{algorithm}
\caption{Semi-supervised Training}
    \label{al2}
  \SetAlgoLined
  \SetKwInOut{Input}{input}\SetKwInOut{Output}{output}
  \Input{
  training set $\tilde{\mathbb{D}}=\{\mathcal{X},\mathcal{\tilde{Y}}\}$, 
  number of samples $N$, 
  number of classes $k$, 
  unsupervised loss coefficient $\lambda_\mathcal{U}$, contrastive loss coefficient $\lambda_\mathcal{C}$, regularization coefficient $\lambda_{r}$,
  network-1 $h_1$ parameters$\theta_1$ and network-2 $h_2$ parameters $\theta_2$ 
  }
  $\theta_1, \theta_2$ = Warmup($\mathcal{X},\mathcal{\tilde{Y}},\theta_1,\theta_2$)\\
  \While {epoch<maxepoch}
{\For {k=1 to 2} 
{$Features_k=g_{(k+1)mod 2}(\tilde{\mathbb{D}},\theta_{(k+1)mod 2})$\\
    $\mathcal{W}$ = $JSD(h(\mathcal{X}),\mathcal{\tilde{Y}})$\\
    $\mathbb{D}_{clean},\mathbb{D}_{noisy}$ = Sample Selection($Features_k$, $\mathcal{W}$)\\
    $\mathcal{X}_k$ = $\{\mathbb{D}_{clean},\mathcal{W}\}$\\
    $\mathcal{U}_k=\mathbb{D}_{noisy}$\\
    \For{iter = 1 to num batchs}{
    From $\{\mathbb{D}_{clean},\mathcal{W}\}$ draw a mini-batch ${(x_b,y_b,w_b);b\in(1,...,B)}$\\
    From $\mathbb{D}_{noisy}$ draw a mini-batch $\{u_b;b\in (1,...,B)\}$\\
    $\mathcal{L}_{\mathcal{X}},\mathcal{L}_{\mathcal{U}}=MixMatch(x_b,u_b)$\\
    Calculate $\mathcal{L}_{\mathcal{C}}$ using eq. \ref{8}\\$\mathcal{L}_{tot}=\mathcal{L}_{\mathcal{X}}+\lambda_{\mathcal{U}}\mathcal{L}_{\mathcal{U}}+\lambda_{r}\mathcal{L}_{reg}+\lambda_\mathcal{C}\mathcal{L}_\mathcal{C}$\\
    $\theta_k = SGD(\mathcal{L}_{total},\theta_k)$
    }
    }
}
\end{algorithm}
\section{Experiments}
\subsection{Datasets}
We evaluate our approach's effectiveness on three benchmark datasets: CIFAR-10, CIFAR-100~\cite{Krizhevsky2009LearningML}, and a real-world dataset, Clothing1M~\cite{Xiao2015LearningFM}, which are introduced as follows:

\textbf{CIFAR-10/100:} The CIFAR-10/100 datasets contain 50k training and 10k test images. respectively. We experiment with two types of noise models: symmetric and asmmetric. In particular, the symmetric noise is generated by randomly replacing the labels to all possible labels with $r$ portion of samples. What's different, the design of asymmetric label noise follows the structure of real mistakes that labels are only replaced by similar classes~(e.g. bird $\rightarrow$ airplane, deer $\rightarrow$ horse). 

\textbf{Clothing1M:} Clothing1M contains 1M clothing images in 14 classes. The dataset has noisy labels due to its origin from multiple online shopping websites, resulting in numerous mislabelled samples. For training, validating, and testing, the dataset has 50k, 14k, and 10k images, respectively.
\subsection{Training Details}
For backbones, the PreAct ResNet18~\cite{he2016identity} architecture is used for CIFAR-10 and CIFAR-100, while the ResNet50~\cite{he2015deep} network is used for Clothing1M. Besides,
the Stochastic Gradient Descent~(SGD) optimization is employed with initial learning rate 0.04, momentum 0.9, weight decay of $5e^{-4}$, and batch size 128 for training on CIFAR-10 and CIFAR-100. The network is trained for 350 epochs with a 10-epoch warm-up for CIFAR-10 and 30-epoch warm-up for CIFAR-100, linearly decaying the learning rate (lr-decay) by 0.1 per 120 epochs. In the case of Clothing1M, the network was trained for 150 epochs with a 2-epoch warm-up using a momentum of 0.9 and a weight decay of $1e^{-3}$. Initially, we set the learning rate to 0.002, which we subsequently reduced by a factor of 10 after 50 and 100 epochs. Moreover, the batch size remained fixed as 32. Auto-augment Policy~\cite{cubuk2019autoaugment} is utilized for data augmentation. In addition, CIFAR10-Policy is used for CIFAR-10 and CIFAR-100, and ImageNet-Policy is used for Clothing1M. Cut off threshold $d_{cutoff}$ for sample selection is set to $0.9$. Hyperparameters for semi-supervised learning $T$ is set to 0.5, $\lambda_{\mathcal{C}}, \lambda_{\mathcal{U}}, \lambda_{R}, \kappa$ are set to $0.025, 30, 1, 0.05$, the beta distribution parameter adopted by MixUp is set to $4$ for CIFAR-10/100 and $0.5$ for Clothing1M.
All the experiments are run on NVIDIA GeForce RTX 3090 GPUs.
\par

    
\subsection{Experimental Results}
To compare with the state-of-the-art approaches, the performance of our method is evaluated under various label noise scenarios. These include synthetic noisy label datasets such as CIFAR-10 and CIFAR-100, as well as real-world noisy datasets like Clothing1M. In particular, the symmetric noise rates of $20\%, 50\%, 80\%$ and asymmetric noise rates of $10\%, 30\%, 40\%$ are considered in our experiments.

Table~\ref{table1} depicts the average performance on the CIFAR-10 and CIFAR-100 datasets under the symmetric noise, in which our method achieves better results than the state-of-the-art approaches. Specifically, for the CIFAR-10 dataset, our PSDC shows superior results than the other methods at medium~($50\%$) and severe~($80\%$) noise levels. Similarly, for the CIFAR-100 dataset, our PSDC shows better performance at low~($20\%$), medium~($50\%$), and high~($80\%$) noise levels. PSDC's accuracy is slightly lower than that of DivideMix~\cite{Li_Richard_Steven:2020} and U{\footnotesize{NI}}C{\footnotesize{ON}}~\cite{Karim_Rizve_Rahnavard:2022} at a 20\% noise rate on the CIFAR-10 dataset. It is possible that this is due to the accumulation of errors caused by inaccurate early feature extraction.


    
\begin{table}[t]
\begin{tabular}{l|ccc|ccc}
\toprule
\multirow{2}{*}{\textbf{Method}} & \multicolumn{3}{c|}{\textbf{CIFAR-10}}        & \multicolumn{3}{c}{\textbf{CIFAR-100}}        \\
                                 & 20\%          & 50\%          & 80\%          & 20\%          & 50\%          & 80\%          \\ \hline
CE                               & 86.8          & 79.4          & 62.9          & 62.0          & 46.7          & 19.9          \\
LDMI~\cite{xu2019ldmi}                             & 88.3          & 81.2          & 43.7          & 58.8          & 51.8          & 27.9          \\
MixUp~\cite{zhang2018mixup}                            & 95.6          & 87.1          & 71.6          & 67.8          & 57.3          & 30.8          \\
Co-teaching+~\cite{yu2019does}                     & 89.5          & 85.7          & 67.4          & 65.6          & 51.8          & 27.9          \\
DivideMix~\cite{Li_Richard_Steven:2020}                        & 96.1 & 94.6          & 92.9          & 77.3          & 74.6          & 60.2          \\
U{\footnotesize{NI}}C{\footnotesize{ON}}~\cite{Karim_Rizve_Rahnavard:2022}                           & 96.0          & 95.6          & 93.9          & 78.9          & 77.6          & 63.9          \\ \hline
(ours)                       & \textbf{96.2}          & \textbf{95.7} & \textbf{94.0} & \textbf{79.4} & \textbf{77.7} & \textbf{64.3} \\ \bottomrule
\end{tabular}
\caption{Test accuracies($\%$) of different methods under symmetric noise on the CIFAR-10 and CIFAR-100 datasets.}
    \label{table1}
\end{table}

\begin{table}[t]
\vspace{-0.7cm}
\begin{tabular}{l|ccc|ccc}
\toprule
\multirow{2}{*}{\textbf{Method}} & \multicolumn{3}{c|}{\textbf{CIFAR-10}}        & \multicolumn{3}{c}{\textbf{CIFAR-100}}        \\
                                 & 10\%          & 30\%          & 40\%          & 10\%          & 30\%          & 40\%          \\ \hline
CE                               & 88.8          & 81.7          & 76.1          & 68.1          & 53.3          & 44.5          \\
LDMI~\cite{xu2019ldmi} & 91.1          & 91.2          & 84.0          & 68.1 &54.1 & 46.2          \\
MixUp~\cite{zhang2018mixup}                            & 93.3          & 83.3          & 77.7          & 72.4          & 57.6          & 48.1          \\
DivideMix~\cite{Li_Richard_Steven:2020}                      & 93.8          & 92.5          & 91.7          & 71.6          & 69.5          & 55.1          \\
MOIT~\cite{ortego2021multiobjective}                       & 94.2 & 94.1          & 93.2          & 77.4          & 75.1          & 74.0          \\
U{\footnotesize{NI}}C{\footnotesize{ON}}~\cite{Karim_Rizve_Rahnavard:2022}                           & 95.3          & 94.8          & 94.1          & 78.2          & 75.6          & 74.8          \\ \hline
(ours)                       & \textbf{95.6}          & \textbf{95.1} & \textbf{94.2} & \textbf{79.1} & \textbf{77.8}/\textbf{80.47} & \textbf{75.1} \\ \bottomrule
\end{tabular}
\caption{Test accuracies($\%$) of different methods under the asymmetric noise on the CIFAR-10 and CIFAR-100 datasets.}
    \label{table2}
\end{table}

\begin{table}[]
\begin{tabular}{llll}
\hline
\multicolumn{4}{c}{\textbf{Tiny-ImageNet}}                                                               \\ \hline
\multicolumn{1}{l|}{Noise(\%)}    & \multicolumn{1}{l}{0}    & \multicolumn{1}{l}{20}   & 50   \\ \hline
\multicolumn{1}{l|}{CE}           & \multicolumn{1}{l}{57.4} & \multicolumn{1}{l}{35.8} & 19.8 \\ \hline
\multicolumn{1}{l|}{Decoupling}   & \multicolumn{1}{l}{-}    & \multicolumn{1}{l}{37.0} & 22.8 \\ \hline
\multicolumn{1}{l|}{F-correction} & \multicolumn{1}{l}{-}    & \multicolumn{1}{l}{44.5} & 33.1 \\ \hline
\multicolumn{1}{l|}{MentorNet}    & \multicolumn{1}{l}{-}    & \multicolumn{1}{l}{45.7} & 35.8 \\ \hline
\multicolumn{1}{l|}{Co-teaching+} & \multicolumn{1}{l}{52.4} & \multicolumn{1}{l}{48.2} & 41.8 \\ \hline
\multicolumn{1}{l|}{M-correction} & \multicolumn{1}{l}{57.7} & \multicolumn{1}{l}{57.2} & 51.6 \\ \hline
\multicolumn{1}{l|}{NCT}          & \multicolumn{1}{l}{62.4} & \multicolumn{1}{l}{58.0} & 47.8 \\ \hline
\multicolumn{1}{l|}{UNICON}       & \multicolumn{1}{l}{62.7} & \multicolumn{1}{l}{59.2} & 52.7 \\ \hline
\multicolumn{1}{l|}{ours}         & \multicolumn{1}{l}{\textbf{63.1}}     & \multicolumn{1}{l}{\textbf{60.9}} & \textbf{53.5}      \\ \hline
\end{tabular}
\end{table}

Table~\ref{table2} presents the average performance on the CIFAR-10 and CIFAR-100 datasets under the asymmetric noise, in which our method achieves better results than all the state-of-the-art approaches. For the CIFAR-100 dataset at a 40\% rate of asymmetric noise rate, the performance of our method is consistent with that of UNICON~\cite{Karim_Rizve_Rahnavard:2022}. This is attributed to the fact that most methods face challenges in learning with high noise rates. In the case of our method, an increase in noise rate elevates the likelihood of clean samples being overwhelmed by noisy samples, resulting in a decrease in performance.

Table~\ref{table3} illustrates the average performance on the Clothing1M dataset. From the results, we can find that our PSDC yields better results than most of the baseline methods, albeit slightly inferior to U{\footnotesize{NI}}C{\footnotesize{ON}}\cite{Karim_Rizve_Rahnavard:2022}. This is likely due to U{\footnotesize{NI}}C{\footnotesize{ON}}'s adoption of a category balance strategy, which enhances the performance of test results. 

\subsection{Ablation Study}
We further study the effect of removing different components, which can offer a better understanding of the factors that contribute to the success of our approach. Without loss of generality, we evaluate our method on the CIFAR-100 dataset for convenience.

\textbf{Effectiveness of Pairwise Similarity Distribution:} To evaluate the performance of pairwise similarity distribution in sample selection, we compare it with the other two sample selection methods under 50\% and 80\% symmetric noise rates. The test accuracies are shown in Table~\ref{table4}, in which:~(1) ``GMM'' means that clustering the samples directly using the extracted features by backbone network; (2) ``GMM+CE'' means that combining the cross-entropy loss with GMM to cluster samples, as done by DivideMix~\cite{Li_Richard_Steven:2020}; and~(3) ``GMM+PSDC'' means that taking the pairwise similarity distribution to represent sample structure and the GMM for dividing the samples into clean set and noisy set. From the results, we can find that ``GMM+PSDC'' can achieve the best results in noisy label learning, which indicates that the combination of GMM and pairwise similarity distribution is a robust sample selection method under different noise rates.

\begin{table}[t]
  \begin{center}
    
    \begin{tabular}{l|c|c} 
    \toprule
      \textbf{Method} & \textbf{Backbone} & \textbf{Test Accuracy}\\
      \hline
      CE & ResNet-50 &69.21\\
      Joint-Optim~\cite{tanaka2018joint}  & ResNet-50 &72.00\\
      MetaCleaner~\cite{Zhang_2019_CVPR} & ResNet-50 &72.50\\
      PCIL~\cite{yi2019probabilistic} & ResNet-50 &73.49\\
      DivideMix~\cite{Li_Richard_Steven:2020} & ResNet-50 &74.76\\
      ELR~\cite{liu2020earlylearning}& ResNet-50 &74.81\\
    U{\footnotesize{NI}}C{\footnotesize{ON}}~\cite{Karim_Rizve_Rahnavard:2022}& ResNet-50 &74.98\\
    CC& ResNet-50 & 75.4\\
      \hline
      ~(ours) &ResNet-50 &\textbf{75.55}\\
    \bottomrule
    \end{tabular}
    \caption{Test accuracies($\%$) of different methods on the Clothing1M dataset.}
    \label{table3}
  \end{center}
  \vspace{-0.8cm}
\end{table}

\begin{table}[t]
  \begin{center}
    
    \begin{tabular}{l|c|c} 
    \toprule
      \textbf{Method} & \textbf{Backbone} & \textbf{Test Accuracy}\\
      \hline
      CE & Vgg19-BN &79.4\\
      	
Nested Dropout & Vgg19-BN &81.3\\
     SELFIE & Vgg19-BN &81.8\\
      	
Nested+Co-teaching
(NCT)& Vgg19-BN &84.1\\
      InstanceGM with ConvNeXt & ConvNeXt	&84.7\\
      Dynamic Loss& Vgg19-BN &86.5\\
   BtR& Vgg19-BN &\textbf{88.5}\\
   
      \hline
      ~(ours) & Vgg19-BN &87.8\\
    \bottomrule
    \end{tabular}
    \caption{Test accuracies($\%$) of different methods on the Clothing1M dataset.}
    \label{table3}
  \end{center}
  \vspace{-0.8cm}
\end{table}

\begin{table}[t]
  \begin{center}
    \begin{tabular}{l|c} 
    \toprule
       \textbf{Method} & \textbf{Test accuracy} \\
    \hline
    \hline
    Dataset & CIFAR-100 \\
    \hline
    Symmetric Noise Rate & $50\%$\quad$80\%$ \\
      \hline
      GMM + PSDC & \textbf{77.7} \quad \textbf{64.3} \\
      GMM + CE & 74.6 \quad 60.2\\
      GMM &74.5 \quad 42.1\\
    \bottomrule
    \end{tabular}
    \caption{Effectiveness of pairwise similarity distribution, in which we take the GMM as sample selection metric and then evaluate  the test accuracy on the CIFAR-100 dataset under different symmetric noise rates.}
    \label{table4}
  \end{center}
  \vspace{-0.8cm}
\end{table}

\begin{table}[t]
  \begin{center}
    \begin{tabular}{l|c} 
    \toprule
       \textbf{Method} & \textbf{Test accuracy} \\
    \hline
    \hline
    Dataset & CIFAR-100 \\
    \hline
    Symmetric Noise Rate & $50\%$\quad$80\%$ \\
      \hline
      PSDC+GMM & \textbf{77.7}\quad \textbf{64.3}\\
      PSDC+K-means& 69.2\quad 36.1\\
    \bottomrule
    \end{tabular}
    \caption{Effectiveness of Gaussian Mixture Model, in which we take the GMM and K-means to cluster samples and then evaluate the test accuracy on the CIFAR-100 dataset under different symmetric noise rates.}
    \label{table6}
  \end{center}
\end{table}

To support our viewpoint, we present the accuracy of selected samples in the clean set under 50\% symmetric noise rate, as shown in Figure~\ref{fig4}, in which the ``Begin'' and ``Highest'' denotes the model obtained after the warmup training and the highest accuracy respectively in the training process. From the results, we find that the lowest accuracy is achieved by simply using the GMM to cluster samples, because it is hard to deal with the high dimensional features without considering the noisy label or sample structure. Besides, the highest accuracy is achieved by jointly using the GMM and pairwise similarity distribution to cluster samples, which indicates that the sample pairwise sample structure is more robust than the noisy label prior information in sample selection at moderate noise rates. Because the pairwise similarity distribution solely relies on the extracted features, and is not directly impacted by noisy labels during sample selection.

\begin{figure}[t]
	\begin{center}
		\includegraphics[width=0.95\linewidth]{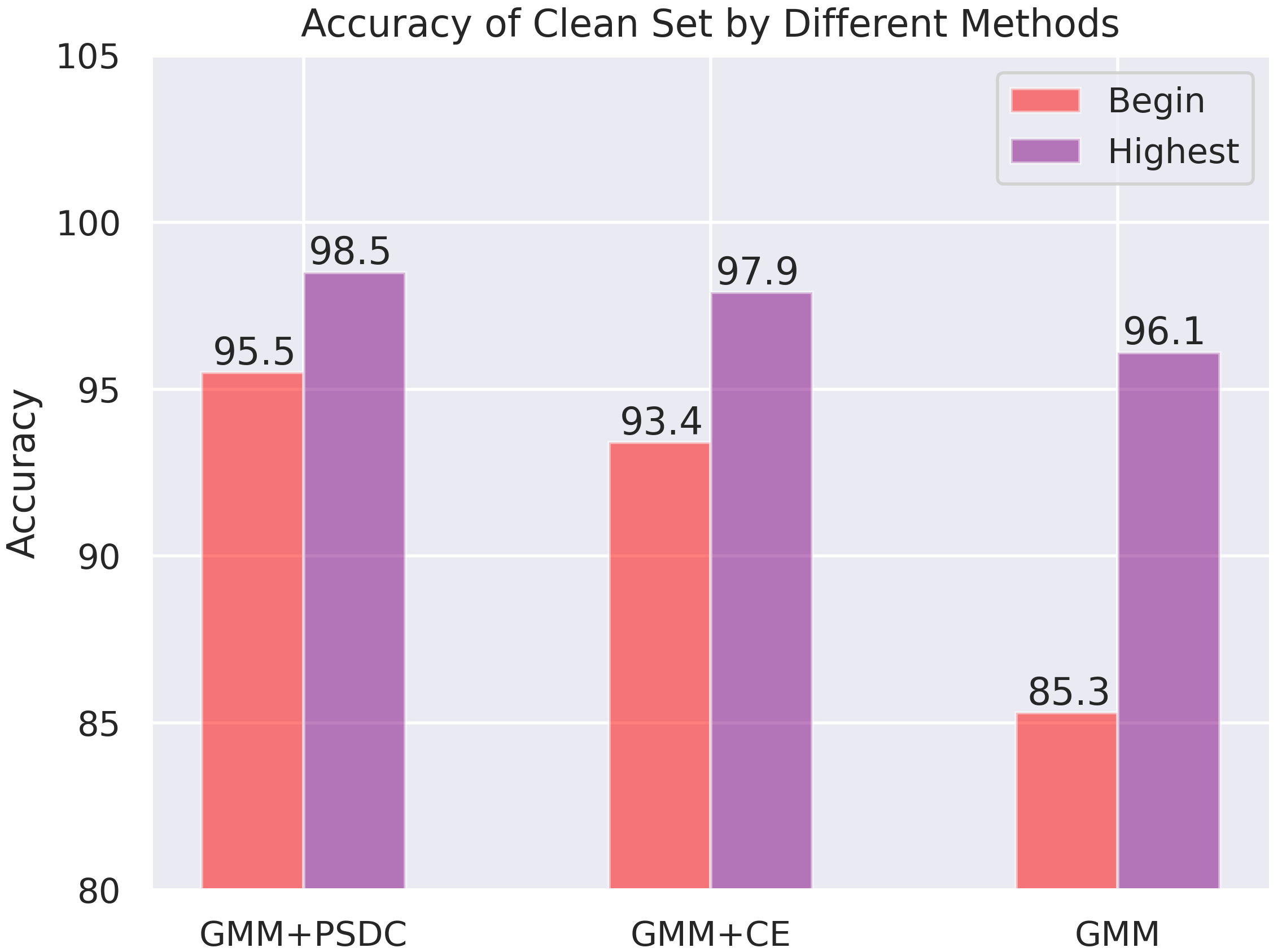}
	\end{center}
	\caption{Accuracy of clean sets using different methods with 50\% symmetric noise added, where samples are clustered using GMM based on extracted features, cross-entropy loss, and pairwise similarity measures, respectively on the CIFAR-100 dataset. }
  \label{fig4}
\end{figure}


\textbf{Effectiveness of Gaussian Mixture Model:} To explain the effectiveness of GMM in sample clustering, we compare it with the widely-used K-means clustering method under 50\% and 80\% symmetric noise rates. It should be noticed that the K-means needs few samples with clean labels to conduct sample selection, while the GMM doesn't need any sample with clean label in the sample selection process. In practice, we give 3 samples with clean labels for K-means to divide the training samples into clean set and noisy set. The test accuracies are shown in Table~\ref{table6}, in which the ``PSDC+GMM'' achieves better results than ``PSDC+K-means'' in both cases. The results indicate that even though no samples with clean labels are used in our GMM, it is superior than the K-means in sample selection. The underlying reason is that the GMM can give the posterior probability to each cluster, therefore the purity of clean set can be significantly improved by setting a threshold in the training process.

\section{Conclusion}
This paper suggests a novel pairwise similarity distribution clustering method for network training with noisy labels. It divides the training samples into one clean set and another noisy set, so as to power any of
the off-the-shelf semi-supervised learning methods to train networks. Unlike the previous methods which take the noisy labels as prior information, we utilize the pairwise similarity distribution as sample structure to increase its adaptability to severe label noise. Our findings demonstrate significant improvements over prior research in various datasets. As future work, we plan to investigate how the pairwise sample structure and noisy label prior can be utilized to complement each other in the context of noisy label learning.
\bibliographystyle{ACM-Reference-Format}
\bibliography{software}


\begin{thebibliography}{47}


\ifx \showCODEN    \undefined \def \showCODEN     #1{\unskip}     \fi
\ifx \showDOI      \undefined \def \showDOI       #1{#1}\fi
\ifx \showISBNx    \undefined \def \showISBNx     #1{\unskip}     \fi
\ifx \showISBNxiii \undefined \def \showISBNxiii  #1{\unskip}     \fi
\ifx \showISSN     \undefined \def \showISSN      #1{\unskip}     \fi
\ifx \showLCCN     \undefined \def \showLCCN      #1{\unskip}     \fi
\ifx \shownote     \undefined \def \shownote      #1{#1}          \fi
\ifx \showarticletitle \undefined \def \showarticletitle #1{#1}   \fi
\ifx \showURL      \undefined \def \showURL       {\relax}        \fi
\providecommand\bibfield[2]{#2}
\providecommand\bibinfo[2]{#2}
\providecommand\natexlab[1]{#1}
\providecommand\showeprint[2][]{arXiv:#2}

\bibitem[Adams(2009)]%
        {adams2009life}
\bibfield{author}{\bibinfo{person}{William~J Adams}.} \bibinfo{year}{2009}\natexlab{}.
\newblock \bibinfo{booktitle}{\emph{The life and times of the central limit theorem}}. Vol.~\bibinfo{volume}{35}.
\newblock \bibinfo{publisher}{American Mathematical Soc.}
\newblock


\bibitem[Arazo et~al\mbox{.}(2019)]%
        {Arazo_Ortego_Albert:2019}
\bibfield{author}{\bibinfo{person}{Eric Arazo}, \bibinfo{person}{Diego Ortego}, \bibinfo{person}{Paul Albert}, \bibinfo{person}{Noel O’Connor}, {and} \bibinfo{person}{Kevin McGuinness}.} \bibinfo{year}{2019}\natexlab{}.
\newblock \showarticletitle{Unsupervised label noise modeling and loss correction}. In \bibinfo{booktitle}{\emph{International conference on machine learning}}. PMLR, \bibinfo{pages}{312--321}.
\newblock


\bibitem[Arpit et~al\mbox{.}(2017)]%
        {article}
\bibfield{author}{\bibinfo{person}{Devansh Arpit}, \bibinfo{person}{Stanisław Jastrzębski}, \bibinfo{person}{Nicolas Ballas}, \bibinfo{person}{David Krueger}, \bibinfo{person}{Emmanuel Bengio}, \bibinfo{person}{Maxinder Kanwal}, \bibinfo{person}{Tegan Maharaj}, \bibinfo{person}{Asja Fischer}, \bibinfo{person}{Aaron Courville}, \bibinfo{person}{Y. Bengio}, {and} \bibinfo{person}{Simon Lacoste-Julien}.} \bibinfo{year}{2017}\natexlab{}.
\newblock \showarticletitle{A Closer Look at Memorization in Deep Networks}.
\newblock  (\bibinfo{date}{06} \bibinfo{year}{2017}).
\newblock


\bibitem[Berthelot et~al\mbox{.}(2019a)]%
        {berthelot2019mixmatch}
\bibfield{author}{\bibinfo{person}{David Berthelot}, \bibinfo{person}{Nicholas Carlini}, \bibinfo{person}{Ian Goodfellow}, \bibinfo{person}{Nicolas Papernot}, \bibinfo{person}{Avital Oliver}, {and} \bibinfo{person}{Colin Raffel}.} \bibinfo{year}{2019}\natexlab{a}.
\newblock \bibinfo{title}{MixMatch: A Holistic Approach to Semi-Supervised Learning}.
\newblock
\newblock
\showeprint[arxiv]{1905.02249}~[cs.LG]


\bibitem[Berthelot et~al\mbox{.}(2019b)]%
        {Berthelot_Carlini_Goodfellow:2019}
\bibfield{author}{\bibinfo{person}{David Berthelot}, \bibinfo{person}{Nicholas Carlini}, \bibinfo{person}{Ian Goodfellow}, \bibinfo{person}{Nicolas Papernot}, \bibinfo{person}{Avital Oliver}, {and} \bibinfo{person}{Colin~A Raffel}.} \bibinfo{year}{2019}\natexlab{b}.
\newblock \showarticletitle{MixMatch: A Holistic Approach to Semi-Supervised Learning}. In \bibinfo{booktitle}{\emph{Advances in Neural Information Processing Systems}}, \bibfield{editor}{\bibinfo{person}{H.~Wallach}, \bibinfo{person}{H.~Larochelle}, \bibinfo{person}{A.~Beygelzimer}, \bibinfo{person}{F.~d\textquotesingle Alch\'{e}-Buc}, \bibinfo{person}{E.~Fox}, {and} \bibinfo{person}{R.~Garnett}} (Eds.), Vol.~\bibinfo{volume}{32}. \bibinfo{publisher}{Curran Associates, Inc.}
\newblock
\urldef\tempurl%
\url{https://proceedings.neurips.cc/paper/2019/file/1cd138d0499a68f4bb72bee04bbec2d7-Paper.pdf}
\showURL{%
\tempurl}


\bibitem[Chen et~al\mbox{.}(2020b)]%
        {chen2020classconditional}
\bibfield{author}{\bibinfo{person}{Pengfei Chen}, \bibinfo{person}{Junjie Ye}, \bibinfo{person}{Guangyong Chen}, \bibinfo{person}{Jingwei Zhao}, {and} \bibinfo{person}{Pheng-Ann Heng}.} \bibinfo{year}{2020}\natexlab{b}.
\newblock \bibinfo{title}{Beyond Class-Conditional Assumption: A Primary Attempt to Combat Instance-Dependent Label Noise}.
\newblock
\newblock
\showeprint[arxiv]{2012.05458}~[cs.LG]


\bibitem[Chen et~al\mbox{.}(2020a)]%
        {chen2020simple}
\bibfield{author}{\bibinfo{person}{Ting Chen}, \bibinfo{person}{Simon Kornblith}, \bibinfo{person}{Mohammad Norouzi}, {and} \bibinfo{person}{Geoffrey Hinton}.} \bibinfo{year}{2020}\natexlab{a}.
\newblock \showarticletitle{A Simple Framework for Contrastive Learning of Visual Representations}.
\newblock \bibinfo{journal}{\emph{arXiv preprint arXiv:2002.05709}} (\bibinfo{year}{2020}).
\newblock


\bibitem[Cubuk et~al\mbox{.}(2019)]%
        {cubuk2019autoaugment}
\bibfield{author}{\bibinfo{person}{Ekin~D. Cubuk}, \bibinfo{person}{Barret Zoph}, \bibinfo{person}{Dandelion Mane}, \bibinfo{person}{Vijay Vasudevan}, {and} \bibinfo{person}{Quoc~V. Le}.} \bibinfo{year}{2019}\natexlab{}.
\newblock \bibinfo{title}{AutoAugment: Learning Augmentation Policies from Data}.
\newblock
\newblock
\showeprint[arxiv]{1805.09501}~[cs.CV]


\bibitem[Ghosh et~al\mbox{.}(2017)]%
        {condition}
\bibfield{author}{\bibinfo{person}{Aritra Ghosh}, \bibinfo{person}{Himanshu Kumar}, {and} \bibinfo{person}{P.~S. Sastry}.} \bibinfo{year}{2017}\natexlab{}.
\newblock \bibinfo{title}{Robust Loss Functions under Label Noise for Deep Neural Networks}.
\newblock
\newblock
\urldef\tempurl%
\url{https://doi.org/10.48550/ARXIV.1712.09482}
\showDOI{\tempurl}


\bibitem[Gui et~al\mbox{.}(2021a)]%
        {gui2021understanding}
\bibfield{author}{\bibinfo{person}{Xian-Jin Gui}, \bibinfo{person}{Wei Wang}, {and} \bibinfo{person}{Zhang-Hao Tian}.} \bibinfo{year}{2021}\natexlab{a}.
\newblock \bibinfo{title}{Towards Understanding Deep Learning from Noisy Labels with Small-Loss Criterion}.
\newblock
\newblock
\showeprint[arxiv]{2106.09291}~[cs.LG]


\bibitem[Gui et~al\mbox{.}(2021b)]%
        {ijcai2021p340}
\bibfield{author}{\bibinfo{person}{Xian-Jin Gui}, \bibinfo{person}{Wei Wang}, {and} \bibinfo{person}{Zhang-Hao Tian}.} \bibinfo{year}{2021}\natexlab{b}.
\newblock \showarticletitle{Towards Understanding Deep Learning from Noisy Labels with Small-Loss Criterion}. In \bibinfo{booktitle}{\emph{Proceedings of the Thirtieth International Joint Conference on Artificial Intelligence, {IJCAI-21}}}, \bibfield{editor}{\bibinfo{person}{Zhi-Hua Zhou}} (Ed.). \bibinfo{publisher}{International Joint Conferences on Artificial Intelligence Organization}, \bibinfo{pages}{2469--2475}.
\newblock
\urldef\tempurl%
\url{https://doi.org/10.24963/ijcai.2021/340}
\showDOI{\tempurl}
\newblock
\shownote{Main Track}.


\bibitem[Han et~al\mbox{.}(2018a)]%
        {coteaching}
\bibfield{author}{\bibinfo{person}{Bo Han}, \bibinfo{person}{Quanming Yao}, \bibinfo{person}{Xingrui Yu}, \bibinfo{person}{Gang Niu}, \bibinfo{person}{Miao Xu}, \bibinfo{person}{Weihua Hu}, \bibinfo{person}{Ivor Tsang}, {and} \bibinfo{person}{Masashi Sugiyama}.} \bibinfo{year}{2018}\natexlab{a}.
\newblock \bibinfo{title}{Co-teaching: Robust Training of Deep Neural Networks with Extremely Noisy Labels}.
\newblock
\newblock
\urldef\tempurl%
\url{https://doi.org/10.48550/ARXIV.1804.06872}
\showDOI{\tempurl}


\bibitem[Han et~al\mbox{.}(2018b)]%
        {han2018coteaching}
\bibfield{author}{\bibinfo{person}{Bo Han}, \bibinfo{person}{Quanming Yao}, \bibinfo{person}{Xingrui Yu}, \bibinfo{person}{Gang Niu}, \bibinfo{person}{Miao Xu}, \bibinfo{person}{Weihua Hu}, \bibinfo{person}{Ivor Tsang}, {and} \bibinfo{person}{Masashi Sugiyama}.} \bibinfo{year}{2018}\natexlab{b}.
\newblock \showarticletitle{Co-teaching: Robust training of deep neural networks with extremely noisy labels}. In \bibinfo{booktitle}{\emph{NeurIPS}}. \bibinfo{pages}{8535--8545}.
\newblock


\bibitem[He et~al\mbox{.}(2015)]%
        {he2015deep}
\bibfield{author}{\bibinfo{person}{Kaiming He}, \bibinfo{person}{Xiangyu Zhang}, \bibinfo{person}{Shaoqing Ren}, {and} \bibinfo{person}{Jian Sun}.} \bibinfo{year}{2015}\natexlab{}.
\newblock \bibinfo{title}{Deep Residual Learning for Image Recognition}.
\newblock
\newblock
\showeprint[arxiv]{1512.03385}~[cs.CV]


\bibitem[He et~al\mbox{.}(2016a)]%
        {He_Zhang_Ren:2016}
\bibfield{author}{\bibinfo{person}{Kaiming He}, \bibinfo{person}{Xiangyu Zhang}, \bibinfo{person}{Shaoqing Ren}, {and} \bibinfo{person}{Jian Sun}.} \bibinfo{year}{2016}\natexlab{a}.
\newblock \showarticletitle{Deep residual learning for image recognition}. In \bibinfo{booktitle}{\emph{Proceedings of the IEEE conference on computer vision and pattern recognition}}. \bibinfo{pages}{770--778}.
\newblock


\bibitem[He et~al\mbox{.}(2016b)]%
        {he2016identity}
\bibfield{author}{\bibinfo{person}{Kaiming He}, \bibinfo{person}{Xiangyu Zhang}, \bibinfo{person}{Shaoqing Ren}, {and} \bibinfo{person}{Jian Sun}.} \bibinfo{year}{2016}\natexlab{b}.
\newblock \bibinfo{title}{Identity Mappings in Deep Residual Networks}.
\newblock
\newblock
\showeprint[arxiv]{1603.05027}~[cs.CV]


\bibitem[Hendrycks et~al\mbox{.}(2018)]%
        {hendrycks2018glc}
\bibfield{author}{\bibinfo{person}{Dan Hendrycks}, \bibinfo{person}{Mantas Mazeika}, \bibinfo{person}{Duncan Wilson}, {and} \bibinfo{person}{Kevin Gimpel}.} \bibinfo{year}{2018}\natexlab{}.
\newblock \showarticletitle{Using Trusted Data to Train Deep Networks on Labels Corrupted by Severe Noise}.
\newblock \bibinfo{journal}{\emph{Advances in Neural Information Processing Systems}} (\bibinfo{year}{2018}).
\newblock


\bibitem[Karim et~al\mbox{.}(2022)]%
        {Karim_Rizve_Rahnavard:2022}
\bibfield{author}{\bibinfo{person}{Nazmul Karim}, \bibinfo{person}{Mamshad~Nayeem Rizve}, \bibinfo{person}{Nazanin Rahnavard}, \bibinfo{person}{Ajmal Mian}, {and} \bibinfo{person}{Mubarak Shah}.} \bibinfo{year}{2022}\natexlab{}.
\newblock \showarticletitle{UniCon: Combating Label Noise Through Uniform Selection and Contrastive Learning}. In \bibinfo{booktitle}{\emph{Proceedings of the IEEE/CVF Conference on Computer Vision and Pattern Recognition (CVPR)}}. \bibinfo{pages}{9676--9686}.
\newblock


\bibitem[Krizhevsky(2009)]%
        {Krizhevsky2009LearningML}
\bibfield{author}{\bibinfo{person}{Alex Krizhevsky}.} \bibinfo{year}{2009}\natexlab{}.
\newblock \showarticletitle{Learning Multiple Layers of Features from Tiny Images}.
\newblock


\bibitem[Li et~al\mbox{.}(2020)]%
        {Li_Richard_Steven:2020}
\bibfield{author}{\bibinfo{person}{Junnan Li}, \bibinfo{person}{Richard Socher}, {and} \bibinfo{person}{Steven~C.H. Hoi}.} \bibinfo{year}{2020}\natexlab{}.
\newblock \showarticletitle{DivideMix: Learning with Noisy Labels as Semi-supervised Learning}. In \bibinfo{booktitle}{\emph{International Conference on Learning Representations}}.
\newblock


\bibitem[Liu et~al\mbox{.}(2020)]%
        {liu2020earlylearning}
\bibfield{author}{\bibinfo{person}{Sheng Liu}, \bibinfo{person}{Jonathan Niles-Weed}, \bibinfo{person}{Narges Razavian}, {and} \bibinfo{person}{Carlos Fernandez-Granda}.} \bibinfo{year}{2020}\natexlab{}.
\newblock \bibinfo{title}{Early-Learning Regularization Prevents Memorization of Noisy Labels}.
\newblock
\newblock
\showeprint[arxiv]{2007.00151}~[cs.LG]


\bibitem[Lu and He(2022)]%
        {lu2022selc}
\bibfield{author}{\bibinfo{person}{Yangdi Lu} {and} \bibinfo{person}{Wenbo He}.} \bibinfo{year}{2022}\natexlab{}.
\newblock \bibinfo{title}{SELC: Self-Ensemble Label Correction Improves Learning with Noisy Labels}.
\newblock
\newblock
\showeprint[arxiv]{2205.01156}~[cs.CV]


\bibitem[Ortego et~al\mbox{.}(2021)]%
        {ortego2021multiobjective}
\bibfield{author}{\bibinfo{person}{Diego Ortego}, \bibinfo{person}{Eric Arazo}, \bibinfo{person}{Paul Albert}, \bibinfo{person}{Noel~E. O'Connor}, {and} \bibinfo{person}{Kevin McGuinness}.} \bibinfo{year}{2021}\natexlab{}.
\newblock \bibinfo{title}{Multi-Objective Interpolation Training for Robustness to Label Noise}.
\newblock
\newblock
\showeprint[arxiv]{2012.04462}~[cs.CV]


\bibitem[Patrini et~al\mbox{.}(2016)]%
        {Patrini2016MakingDN}
\bibfield{author}{\bibinfo{person}{Giorgio Patrini}, \bibinfo{person}{Alessandro Rozza}, \bibinfo{person}{Aditya~Krishna Menon}, \bibinfo{person}{Richard Nock}, {and} \bibinfo{person}{Lizhen Qu}.} \bibinfo{year}{2016}\natexlab{}.
\newblock \showarticletitle{Making Deep Neural Networks Robust to Label Noise: A Loss Correction Approach}.
\newblock \bibinfo{journal}{\emph{2017 IEEE Conference on Computer Vision and Pattern Recognition (CVPR)}} (\bibinfo{year}{2016}), \bibinfo{pages}{2233--2241}.
\newblock


\bibitem[Reed et~al\mbox{.}(2015)]%
        {reed2015training}
\bibfield{author}{\bibinfo{person}{Scott Reed}, \bibinfo{person}{Honglak Lee}, \bibinfo{person}{Dragomir Anguelov}, \bibinfo{person}{Christian Szegedy}, \bibinfo{person}{Dumitru Erhan}, {and} \bibinfo{person}{Andrew Rabinovich}.} \bibinfo{year}{2015}\natexlab{}.
\newblock \bibinfo{title}{Training Deep Neural Networks on Noisy Labels with Bootstrapping}.
\newblock
\newblock
\showeprint[arxiv]{1412.6596}~[cs.CV]


\bibitem[Ren et~al\mbox{.}(2018)]%
        {Ren_Zeng_Yang:2018}
\bibfield{author}{\bibinfo{person}{Mengye Ren}, \bibinfo{person}{Wenyuan Zeng}, \bibinfo{person}{Bin Yang}, {and} \bibinfo{person}{Raquel Urtasun}.} \bibinfo{year}{2018}\natexlab{}.
\newblock \showarticletitle{Learning to reweight examples for robust deep learning}. In \bibinfo{booktitle}{\emph{International conference on machine learning}}. PMLR, \bibinfo{pages}{4334--4343}.
\newblock


\bibitem[Shu et~al\mbox{.}(2019)]%
        {Shu_Xie_Yi:2019}
\bibfield{author}{\bibinfo{person}{Jun Shu}, \bibinfo{person}{Qi Xie}, \bibinfo{person}{Lixuan Yi}, \bibinfo{person}{Qian Zhao}, \bibinfo{person}{Sanping Zhou}, \bibinfo{person}{Zongben Xu}, {and} \bibinfo{person}{Deyu Meng}.} \bibinfo{year}{2019}\natexlab{}.
\newblock \showarticletitle{Meta-weight-net: Learning an explicit mapping for sample weighting}.
\newblock \bibinfo{journal}{\emph{Advances in neural information processing systems}}  \bibinfo{volume}{32} (\bibinfo{year}{2019}).
\newblock


\bibitem[Song et~al\mbox{.}(2019)]%
        {pmlr-v97-song19b}
\bibfield{author}{\bibinfo{person}{Hwanjun Song}, \bibinfo{person}{Minseok Kim}, {and} \bibinfo{person}{Jae-Gil Lee}.} \bibinfo{year}{2019}\natexlab{}.
\newblock \showarticletitle{{SELFIE}: Refurbishing Unclean Samples for Robust Deep Learning}. In \bibinfo{booktitle}{\emph{Proceedings of the 36th International Conference on Machine Learning}} \emph{(\bibinfo{series}{Proceedings of Machine Learning Research}, Vol.~\bibinfo{volume}{97})}, \bibfield{editor}{\bibinfo{person}{Kamalika Chaudhuri} {and} \bibinfo{person}{Ruslan Salakhutdinov}} (Eds.). \bibinfo{publisher}{PMLR}, \bibinfo{pages}{5907--5915}.
\newblock
\urldef\tempurl%
\url{https://proceedings.mlr.press/v97/song19b.html}
\showURL{%
\tempurl}


\bibitem[Song et~al\mbox{.}(2022)]%
        {Song_Kim_Park:2022}
\bibfield{author}{\bibinfo{person}{Hwanjun Song}, \bibinfo{person}{Minseok Kim}, \bibinfo{person}{Dongmin Park}, \bibinfo{person}{Yooju Shin}, {and} \bibinfo{person}{Jae-Gil Lee}.} \bibinfo{year}{2022}\natexlab{}.
\newblock \showarticletitle{Learning from noisy labels with deep neural networks: A survey}.
\newblock \bibinfo{journal}{\emph{IEEE Transactions on Neural Networks and Learning Systems}} (\bibinfo{year}{2022}).
\newblock


\bibitem[Tanaka et~al\mbox{.}(2018)]%
        {tanaka2018joint}
\bibfield{author}{\bibinfo{person}{Daiki Tanaka}, \bibinfo{person}{Daiki Ikami}, \bibinfo{person}{Toshihiko Yamasaki}, {and} \bibinfo{person}{Kiyoharu Aizawa}.} \bibinfo{year}{2018}\natexlab{}.
\newblock \bibinfo{title}{Joint Optimization Framework for Learning with Noisy Labels}.
\newblock
\newblock
\showeprint[arxiv]{1803.11364}~[cs.CV]


\bibitem[Tarvainen and Valpola(2018)]%
        {tarvainen2018mean}
\bibfield{author}{\bibinfo{person}{Antti Tarvainen} {and} \bibinfo{person}{Harri Valpola}.} \bibinfo{year}{2018}\natexlab{}.
\newblock \bibinfo{title}{Mean teachers are better role models: Weight-averaged consistency targets improve semi-supervised deep learning results}.
\newblock
\newblock
\showeprint[arxiv]{1703.01780}~[cs.NE]


\bibitem[Wang et~al\mbox{.}(2021a)]%
        {Wang_Hua_Kodirov:2021}
\bibfield{author}{\bibinfo{person}{Xinshao Wang}, \bibinfo{person}{Yang Hua}, \bibinfo{person}{Elyor Kodirov}, \bibinfo{person}{David~A Clifton}, {and} \bibinfo{person}{Neil~M Robertson}.} \bibinfo{year}{2021}\natexlab{a}.
\newblock \showarticletitle{Proselflc: Progressive self label correction for training robust deep neural networks}. In \bibinfo{booktitle}{\emph{Proceedings of the IEEE/CVF Conference on Computer Vision and Pattern Recognition}}. \bibinfo{pages}{752--761}.
\newblock


\bibitem[Wang et~al\mbox{.}(2021b)]%
        {Wang_Li_Guo:2021}
\bibfield{author}{\bibinfo{person}{Zhenyu Wang}, \bibinfo{person}{Ya-Li Li}, \bibinfo{person}{Ye Guo}, {and} \bibinfo{person}{Shengjin Wang}.} \bibinfo{year}{2021}\natexlab{b}.
\newblock \showarticletitle{Combating noise: semi-supervised learning by region uncertainty quantification}.
\newblock \bibinfo{journal}{\emph{Advances in Neural Information Processing Systems}}  \bibinfo{volume}{34} (\bibinfo{year}{2021}), \bibinfo{pages}{9534--9545}.
\newblock


\bibitem[Wu et~al\mbox{.}(2021)]%
        {Wu_Shu_Xie:2021}
\bibfield{author}{\bibinfo{person}{Yichen Wu}, \bibinfo{person}{Jun Shu}, \bibinfo{person}{Qi Xie}, \bibinfo{person}{Qian Zhao}, {and} \bibinfo{person}{Deyu Meng}.} \bibinfo{year}{2021}\natexlab{}.
\newblock \showarticletitle{Learning to purify noisy labels via meta soft label corrector}. In \bibinfo{booktitle}{\emph{Proceedings of the AAAI Conference on Artificial Intelligence}}, Vol.~\bibinfo{volume}{35}. \bibinfo{pages}{10388--10396}.
\newblock


\bibitem[Xiao et~al\mbox{.}(2015)]%
        {Xiao2015LearningFM}
\bibfield{author}{\bibinfo{person}{Tong Xiao}, \bibinfo{person}{Tian Xia}, \bibinfo{person}{Yi Yang}, \bibinfo{person}{Chang Huang}, {and} \bibinfo{person}{Xiaogang Wang}.} \bibinfo{year}{2015}\natexlab{}.
\newblock \showarticletitle{Learning from massive noisy labeled data for image classification}.
\newblock \bibinfo{journal}{\emph{2015 IEEE Conference on Computer Vision and Pattern Recognition (CVPR)}} (\bibinfo{year}{2015}), \bibinfo{pages}{2691--2699}.
\newblock


\bibitem[Xu et~al\mbox{.}(2019)]%
        {xu2019ldmi}
\bibfield{author}{\bibinfo{person}{Yilun Xu}, \bibinfo{person}{Peng Cao}, \bibinfo{person}{Yuqing Kong}, {and} \bibinfo{person}{Yizhou Wang}.} \bibinfo{year}{2019}\natexlab{}.
\newblock \bibinfo{title}{$L_{DMI}$: An Information-theoretic Noise-robust Loss Function}.
\newblock
\newblock
\showeprint[arxiv]{1909.03388}~[cs.LG]


\bibitem[Yi and Wu(2019)]%
        {yi2019probabilistic}
\bibfield{author}{\bibinfo{person}{Kun Yi} {and} \bibinfo{person}{Jianxin Wu}.} \bibinfo{year}{2019}\natexlab{}.
\newblock \bibinfo{title}{Probabilistic End-to-end Noise Correction for Learning with Noisy Labels}.
\newblock
\newblock
\showeprint[arxiv]{1903.07788}~[cs.CV]


\bibitem[Yu et~al\mbox{.}(2019)]%
        {yu2019does}
\bibfield{author}{\bibinfo{person}{Xingrui Yu}, \bibinfo{person}{Bo Han}, \bibinfo{person}{Jiangchao Yao}, \bibinfo{person}{Gang Niu}, \bibinfo{person}{Ivor~W. Tsang}, {and} \bibinfo{person}{Masashi Sugiyama}.} \bibinfo{year}{2019}\natexlab{}.
\newblock \bibinfo{title}{How does Disagreement Help Generalization against Label Corruption?}
\newblock
\newblock
\showeprint[arxiv]{1901.04215}~[cs.LG]


\bibitem[Zhang et~al\mbox{.}(2021)]%
        {Zhang_Jiang_Hou:2021}
\bibfield{author}{\bibinfo{person}{Chang-Bin Zhang}, \bibinfo{person}{Peng-Tao Jiang}, \bibinfo{person}{Qibin Hou}, \bibinfo{person}{Yunchao Wei}, \bibinfo{person}{Qi Han}, \bibinfo{person}{Zhen Li}, {and} \bibinfo{person}{Ming-Ming Cheng}.} \bibinfo{year}{2021}\natexlab{}.
\newblock \showarticletitle{Delving deep into label smoothing}.
\newblock \bibinfo{journal}{\emph{IEEE Transactions on Image Processing}}  \bibinfo{volume}{30} (\bibinfo{year}{2021}), \bibinfo{pages}{5984--5996}.
\newblock


\bibitem[Zhang et~al\mbox{.}(2018)]%
        {zhang2018mixup}
\bibfield{author}{\bibinfo{person}{Hongyi Zhang}, \bibinfo{person}{Moustapha Cisse}, \bibinfo{person}{Yann~N. Dauphin}, {and} \bibinfo{person}{David Lopez-Paz}.} \bibinfo{year}{2018}\natexlab{}.
\newblock \bibinfo{title}{mixup: Beyond Empirical Risk Minimization}.
\newblock
\newblock
\showeprint[arxiv]{1710.09412}~[cs.LG]


\bibitem[Zhang et~al\mbox{.}(2019)]%
        {Zhang_2019_CVPR}
\bibfield{author}{\bibinfo{person}{Weihe Zhang}, \bibinfo{person}{Yali Wang}, {and} \bibinfo{person}{Yu Qiao}.} \bibinfo{year}{2019}\natexlab{}.
\newblock \showarticletitle{MetaCleaner: Learning to Hallucinate Clean Representations for Noisy-Labeled Visual Recognition}. In \bibinfo{booktitle}{\emph{Proceedings of the IEEE/CVF Conference on Computer Vision and Pattern Recognition (CVPR)}}.
\newblock


\bibitem[Zhang et~al\mbox{.}(2020)]%
        {Zhang_Wu_Chen:2020}
\bibfield{author}{\bibinfo{person}{Xuchao Zhang}, \bibinfo{person}{Xian Wu}, \bibinfo{person}{Fanglan Chen}, \bibinfo{person}{Liang Zhao}, {and} \bibinfo{person}{Chang-Tien Lu}.} \bibinfo{year}{2020}\natexlab{}.
\newblock \showarticletitle{Self-paced robust learning for leveraging clean labels in noisy data}. In \bibinfo{booktitle}{\emph{Proceedings of the AAAI conference on artificial intelligence}}, Vol.~\bibinfo{volume}{34}. \bibinfo{pages}{6853--6860}.
\newblock


\bibitem[Zheng et~al\mbox{.}(2021a)]%
        {Zheng_Awadallah_Dumais:2021}
\bibfield{author}{\bibinfo{person}{Guoqing Zheng}, \bibinfo{person}{Ahmed~Hassan Awadallah}, {and} \bibinfo{person}{Susan Dumais}.} \bibinfo{year}{2021}\natexlab{a}.
\newblock \showarticletitle{Meta label correction for noisy label learning}. In \bibinfo{booktitle}{\emph{Proceedings of the AAAI Conference on Artificial Intelligence}}, Vol.~\bibinfo{volume}{35}. \bibinfo{pages}{11053--11061}.
\newblock


\bibitem[Zheng et~al\mbox{.}(2021b)]%
        {zheng2021meta}
\bibfield{author}{\bibinfo{person}{Guoqing Zheng}, \bibinfo{person}{Ahmed~Hassan Awadallah}, {and} \bibinfo{person}{Susan Dumais}.} \bibinfo{year}{2021}\natexlab{b}.
\newblock \bibinfo{title}{Meta Label Correction for Noisy Label Learning}.
\newblock
\newblock
\showeprint[arxiv]{1911.03809}~[cs.LG]


\bibitem[Zhou et~al\mbox{.}(2021)]%
        {Zhou_Wang_Shu:2021}
\bibfield{author}{\bibinfo{person}{Sanping Zhou}, \bibinfo{person}{Jinjun Wang}, \bibinfo{person}{Jun Shu}, \bibinfo{person}{Deyu Meng}, \bibinfo{person}{Le Wang}, {and} \bibinfo{person}{Nanning Zheng}.} \bibinfo{year}{2021}\natexlab{}.
\newblock \showarticletitle{Multinetwork collaborative feature learning for semisupervised person reidentification}.
\newblock \bibinfo{journal}{\emph{IEEE Transactions on Neural Networks and Learning Systems}} \bibinfo{volume}{33}, \bibinfo{number}{9} (\bibinfo{year}{2021}), \bibinfo{pages}{4826--4839}.
\newblock


\bibitem[Zhou et~al\mbox{.}(2017)]%
        {Zhou_Wang_Wang:2017}
\bibfield{author}{\bibinfo{person}{Sanping Zhou}, \bibinfo{person}{Jinjun Wang}, \bibinfo{person}{Jiayun Wang}, \bibinfo{person}{Yihong Gong}, {and} \bibinfo{person}{Nanning Zheng}.} \bibinfo{year}{2017}\natexlab{}.
\newblock \showarticletitle{Point to set similarity based deep feature learning for person re-identification}. In \bibinfo{booktitle}{\emph{Proceedings of the IEEE Conference on Computer Vision and Pattern Recognition}}. \bibinfo{pages}{3741--3750}.
\newblock


\bibitem[Zhou et~al\mbox{.}(2023)]%
        {Zhou_Wang_Wang:2023}
\bibfield{author}{\bibinfo{person}{Sanping Zhou}, \bibinfo{person}{Jinjun Wang}, \bibinfo{person}{Le Wang}, \bibinfo{person}{Xingyu Wan}, \bibinfo{person}{Siqi Hui}, {and} \bibinfo{person}{Nanning Zheng}.} \bibinfo{year}{2023}\natexlab{}.
\newblock \showarticletitle{Inverse Adversarial Diversity Learning for Network Ensemble}.
\newblock \bibinfo{journal}{\emph{IEEE Transactions on Neural Networks and Learning Systems}} (\bibinfo{year}{2023}).
\newblock


\end{thebibliography}

\end{document}



\maketitle

\section{Proof}
\begin{theorem}
Consider two pairs of samples, ${(x_p,\tilde{y}),(x_q,\tilde{y})}$, randomly selected from $\mathbb{G}_i$
, with their respective indices in the affinity matrix $A^i$ being $p$ and $q$. Given the following conditions:
\begin{enumerate}
\item $\mathcal{C}(x_p) \neq \tilde{y}$ and $\mathcal{C}(x_q) = \tilde{y}$;
\item clean samples are not submerged by noise samples;
\item clean samples and noise samples obey different distributions.
\end{enumerate}
Then the mean value of row $p$ on the affinity matrix $A^i$ follows a Gaussian distribution with mean $\mu_p$ and the mean value of row $q$ in the affinity matrix $A^i$ follows a Gaussian distribution with mean $\mu_q$, where $\mu_q < \mu_p$.  
\end{theorem}
\begin{proof}
 Consider the sum:
 \begin{equation}
     S_1 = a_{p_1} + a_{p_2} + ... + a_{p_n},
 \end{equation}
  \begin{equation}
     S_2 = a_{q_1} + a_{q_2} + ... + a_{q_{n}}.
 \end{equation}
 We normalize it to obtain a random variable with zero mean and unit variance as follows:
 \begin{equation}
     Z_p = \frac{S_1-E(S_1)}{\sqrt{Var(S_1)}} = \frac{1}{\sigma_p}\sum_{i=p_1}^{p_n}(a_{i}-\mu_p).
 \end{equation}
 Then, according to the Lindeberg-Levy central limit theorem, as $N\rightarrow \infty, Z_p\rightarrow N(0,1)$ in distribution. 
 For the noise sample, their average affinity are a sequence of independent random variables but not necessarily uniformly distributed, we assume that it satisfies the Liapunov Condition:
 $a_q = \{a_{q_1},...,a_{q_{n}}\}$ is an independent sequence of random variables satisfying $E(a_q)=\mu_q, Var(a_q) = \sigma_q<\infty$.For some $\delta>0$
 \begin{equation}
     \frac{\sum_{k=1}^{n}E|X_{k}-\mu_q|^{2+\delta}}{\sigma_q^{2+\delta}}\rightarrow0.
 \end{equation}
Hence according to the Lyapunov central limit theorem:
 \begin{equation}
     Z_q = \frac{1}{\sigma_q}\sum_{i=q_1}^{q_{n}}(a_{i}-\mu_q).
 \end{equation}
 $Z_q\rightarrow N(0,1)$ as $N\rightarrow \infty$.
 Because clean samples are not submerged by noise samples, we know
 $S_1>S_2$ and we have proofed $S_1,S_2$ obey a gaussian distribution respectively with mean $\mu_p$ and $\mu_q$, so we get:
 \begin{equation}
     \mu_p>\mu_q
 \end{equation}
\end{proof}
\begin{theorem}
Consider two pairs of randomly selected samples, $(\boldsymbol{x}_1,\tilde{y})$ and $(\boldsymbol{x}_2,\tilde{y})$, with the same observed label from the set $\{\mathcal{X},\mathcal{\tilde{Y}}\}$. If the following conditions hold:
\begin{enumerate}
    \item $\mathcal{C}(\boldsymbol{x}_1)=\Tilde{y}$ and $\mathcal{C}(\boldsymbol{x}_2)\neq\Tilde{y}$;
    \item the noise transition matrix T of $\{\mathcal{X},\mathcal{\tilde{Y}}\}$ satisfies the diagonally-dominant condition $T_{ii}>max\{max_{j\neq i}T_{ij},max_{j\neq i}T_{ji}\}$,$\forall i$;
    \item noise type is uniform and noise rate is under 1, or noise type is pairwise and noise rate is under 0.5, or noise type is structured and noise rate is under 0.5. 
\end{enumerate}
are satisfied, then $JSD(h(\boldsymbol{x}_1),\Tilde{y})<JSD(h(\boldsymbol{x}_2),\Tilde{y})$, where JSD denotes Jensen-Shannon divergence. 
\end{theorem}
\begin{proof}
Consider the definition of JSD, for an example $(\boldsymbol{x},y)$

\begin{equation}
\begin{aligned}
& JSD(h(\boldsymbol{x}), y) \\
& =\frac{1}{2} KLD\left(h(\boldsymbol{x}) \| \frac{h(\boldsymbol{x})+y}{2}\right)+\frac{1}{2}KLD\left(y \| \frac{h(\boldsymbol{x})+y}{2}\right) \\
& =\frac{1}{2} \sum_{i \neq y}^{c-1} \hat{p}_{i}(\boldsymbol{x}) \log \frac{2 \hat{p}_{i}(\boldsymbol{x})}{\hat{p}_{i}(\boldsymbol{x})}+\frac{1}{2} p_{y}(\boldsymbol{x}) \log \frac{2 p_{y}(\boldsymbol{x})}{1+p_{y}(\boldsymbol{x})}+\frac{1}{2} \log \frac{2}{1+p_{y}(\boldsymbol{x})} \\
& =\frac{1}{2} \sum_{i \neq y}^{c-1} T_{h(\boldsymbol{x}), i} \log 2+\frac{1}{2} T_{h(\boldsymbol{x}), y} \log \frac{2 T_{h(\boldsymbol{x}), y}}{1+T_{h(\boldsymbol{x}), y}}+\frac{1}{2} \log \frac{2}{1+T_{h(\boldsymbol{x}), y}} \\
& =\frac{1}{2}\left(\log 2 (1+\sum_{i \neq y}^{c-1} T_{h(\boldsymbol{x}), i})+T_{h(\boldsymbol{x}), y} \log \frac{2 T_{h(\boldsymbol{x}), y}}{1+T_{h(\boldsymbol{x}), y}}-\log \left(1+T_{h(\boldsymbol{x}), y}\right)\right)
\end{aligned} 
\end{equation}


In order to proof $J S D\left(h\left(\boldsymbol{x}_{2}\right), \Tilde{y}\right)>J S D\left(h\left(\boldsymbol{x}_{1}\right), \Tilde{y}\right)$, we need to proof
\begin{equation}\label{3}
     \sum_{i \neq \Tilde{y}}^{c-1} T_{h\left(\boldsymbol{x}_{1}\right), i}< \sum_{i \neq \Tilde{y}}^{c-1} T_{h\left(\boldsymbol{x}_{2}\right), i}
\end{equation}
and
\begin{equation}
\begin{array}{l}\label{4}
     T_{h\left(\boldsymbol{x}_{1}\right), \Tilde{y}} \log \frac{2 T_{h\left(\boldsymbol{x}_{1}\right), \Tilde{y}}}{1+T_{h\left(\boldsymbol{x}_{1}\right), \Tilde{y}}}-\log \left(1+T_{h\left(\boldsymbol{x}_{1}\right), \Tilde{y}}\right) \\
    <T_{h\left(\boldsymbol{x}_{2}\right), \Tilde{y}} \log \frac{2 T_{h\left(\boldsymbol{x}_{2}\right), \Tilde{y}}}{1+T_{h\left(\boldsymbol{x}_{2}\right), \Tilde{y}}}-\log \left(1+T_{h\left(\boldsymbol{x}_{2}\right), \Tilde{y}}\right)
\end{array}
\end{equation}
For inequality (\ref{3}), We analyze it on different noise types. Let r stand for noise rate and the class number is c. \\
On uniform noise, the noise transition matrix T is:
\begin{equation*}
	\begin{bmatrix} 
	   1-\frac{(c-1)r}{c} & \frac{r}{c} &\cdots& \frac{r}{c} \\
	\frac{r}{c}& 1-\frac{(c-1)r}{c}&\cdots & \frac{r}{c}\\
	\vdots&\vdots& \ddots&\vdots \\
	\frac{r}{c}&\frac{r}{c}&\cdots& 1-\frac{(c-1)r}{c}
	\end{bmatrix}
\end{equation*}

So $\sum_{i \neq \Tilde{y}}^{c-1} T_{h\left(\boldsymbol{x}_{1}\right), i}=\frac{(c-1)r}{c}$, and $\sum_{i \neq \Tilde{y}}^{c-1} T_{h\left(\boldsymbol{x}_{2}\right), i}=1-\frac{r}{c}$. Now we get:
\begin{equation}\label{5}
    \sum_{i \neq y}^{c-1} T_{h\left(\boldsymbol{x}_{2}\right), i}-\sum_{i \neq y}^{c-1} T_{h\left(\boldsymbol{x}_{1}\right), i}=1-r
\end{equation}
From equation (\ref{5}) we know that if the noise type is uniform noise and the noise rate $r < 1$, inequality (\ref{3}) is true.\\
On pairwise noise, the noise transition matrix T is:
\begin{equation*}
	\begin{bmatrix} 
	   1-r & r & 0&\dots & 0  \\
	0 & 1-r&r &\dots & 0\\
	\vdots&\vdots& \ddots&\ddots&\vdots \\
	r& 0&\dots& 0 & 1-r
	\end{bmatrix}
\end{equation*}
In this case, $\sum_{i \neq \Tilde{y}}^{c-1} T_{h\left(\boldsymbol{x}_{1}\right), i}= r $, $\sum_{i \neq \Tilde{y}}^{c-1} T_{h\left(\boldsymbol{x}_{2}\right), i}=1-r \text{ or } 1$. We get:
\begin{equation}\label{6}
    \sum_{i \neq y}^{c-1} T_{h\left(\boldsymbol{x}_{2}\right), i}-\sum_{i \neq y}^{c-1} T_{h\left(\boldsymbol{x}_{1}\right), i}=1-r \text { or } 1-2 r
\end{equation}
From equation (\ref{6}) we know that if the noise type is uniform noise and the noise rate $r < 0.5$, inequality (\ref{3}) is true.\\

On structured noise, the noise transition matrix T is:
\begin{equation*}
	\begin{bmatrix} 
	   1 & 0 & 0&\dots & 0& 0  \\
	0 & 1&r&0 &\dots & 0\\
        r&0&1-r&0&\dots&0\\
	\vdots&\vdots& \ddots&\ddots&\ddots&\vdots \\
	0& r&0&\dots& 0 & 1-r
	\end{bmatrix}
\end{equation*}
In this case, $\sum_{i \neq \Tilde{y}}^{c-1} T_{h\left(\boldsymbol{x}_{1}\right), i}= 0\text{ or }r $, $\sum_{i \neq \Tilde{y}}^{c-1} T_{h\left(\boldsymbol{x}_{2}\right), i}=1-r \text{ or } 1$.
\begin{equation}\label{7}
    \sum_{i \neq y}^{c-1} T_{h\left(\boldsymbol{x}_{2}\right), i}-\sum_{i \neq y}^{c-1} T_{h\left(\boldsymbol{x}_{1}\right), i}=1-r \text { or } 1-2 r \text { or } 1
\end{equation}
From equation (\ref{7}) we know that if the noise type is uniform noise and the noise rate $r < 0.5$, inequality (\ref{3}) is true.\\
Now we consider inequality (\ref{4}), think of function $h(x) = xln\frac{2x}{1+x}-ln(1+x),x\in(0,1]$, 
\begin{equation}
    h^{'}(x) = ln(2x)-ln(1+x)
\end{equation}
\begin{equation}
    h^{''}(x) = \frac{1}{x(1+x)}>0
\end{equation}
So $f^{'}(x)$ monotonic increase in $(0,1]$, and having a zero point 1, that means $f^{'}(x)\leq 0$ on $(0,1]$, in other words, $f(x)$ monotonic decrement on $(0,1]$. Now we know that if $T_{f\left(\boldsymbol{x}_{i}\right), \mathrm{j}_{1}}>T_{h\left(\boldsymbol{x}_{i}\right), \mathrm{j}_{2}}$, 
$h(T_{f\left(\boldsymbol{x}_{i}\right), \mathrm{j}_{1}})<h(T_{f\left(\boldsymbol{x}_{i}\right), \mathrm{j}_{2}})$ because $T_{f\left(\boldsymbol{x}_{i}\right), \mathrm{j}}\in(0,1]$.\\
In addition, 
\begin{equation}
\begin{aligned}
    &T_{h\left(\boldsymbol{x}_{1}\right), \Tilde{y}} \log \frac{2 T_{h\left(\boldsymbol{x}_{1}\right),\Tilde{y}}}{1+T_{f\left(\boldsymbol{x}_{1}\right),\Tilde{y}}}-\log \left(1+T_{h\left(\boldsymbol{x}_{1}\right), \Tilde{y}}\right)\\
    &=T_{\Tilde{y}, \Tilde{y}} \log \frac{2 T_{\Tilde{y}, \Tilde{y}}}{1+T_{\Tilde{y}, \Tilde{y}}}-\log \left(1+T_{\Tilde{y},\Tilde{y}}\right)
\end{aligned}
\end{equation}
Under the diagonally-dominant condition 
\begin{equation}
    T_{ii}>max\{max_{j\neq i}T_{ij},max_{j\neq i}T_{ji}\},
\end{equation}
we can det inequality (\ref{4}) is true.
Now we have proof the inequality (\ref{3}) and \ref{4}, so the theorem is valid. 
\end{proof}